%% file: arxiv_format/arxiv_main.tex
\documentclass{article}
\usepackage[margin=1.2in]{geometry}

\usepackage[utf8]{inputenc} % allow utf-8 input
\usepackage[T1]{fontenc}    % use 8-bit T1 fonts
\usepackage{hyperref}       % hyperlinks
\usepackage{url}            % simple URL typesetting
\usepackage{booktabs}       % professional-quality tables
\usepackage{amsfonts}       % blackboard math symbols
\usepackage{nicefrac}       % compact symbols for 1/2, etc.
\usepackage{microtype}      % microtypography
\usepackage{lipsum}
\usepackage{graphicx}
\graphicspath{ {./images/} }

\usepackage{natbib}
\usepackage{amsmath}
\usepackage{amssymb}
\usepackage{amsthm}
\usepackage{algorithm}
\usepackage{algpseudocode}

\usepackage{color-edits}
\addauthor{dmt}{blue}

\newcommand{\Pb}{\mathbb{P}}
\newcommand{\Nb}{\mathbb{N}}
\newcommand{\R}{\mathbb{R}}
\newcommand{\E}{\mathbb{E}}
\newcommand{\V}{\mathbb{V}}
\newcommand{\X}{\mathcal{X}}

\newcommand{\F}{\mathcal{F}}
\newcommand{\bigo}{\mathcal{O}}

\newcommand{\lp}{\left(}
\newcommand{\rp}{\right)}
\newcommand{\lc}{\left\{}
\newcommand{\rc}{\right\}}
\newcommand{\lb}{\left[}
\newcommand{\rb}{\right]}
\newcommand{\rba}{\right|}
\newcommand{\lba}{\left|}
\newcommand{\rdba}{\right\|}
\newcommand{\ldba}{\left\|}
\newcommand{\ra}{\right\rangle}
\newcommand{\la}{\left\langle}

\newtheorem{theorem}{Theorem}

\newtheorem{proposition}{Proposition}
\newtheorem{corollary}{Corollary}

\newtheorem{assumption}{Assumption}
\newtheorem{fact}{Fact}

\usepackage[colorinlistoftodos]{todonotes}

\title{Vector-valued self-normalized concentration inequalities beyond sub-Gaussianity} %Self-Normalized Processes beyond Sub-Gaussianity

\author{
 Diego Martinez-Taboada$^{1}$, Tomas Gonzalez$^{2}$, and Aaditya Ramdas$^{12}$\\
$^1$Department of Statistics \& Data Science\\
$^2$Machine Learning Department \\
Carnegie Mellon University\\
\texttt{\{diegomar,tcgonzal,aramdas\}@andrew.cmu.edu} 
}

\begin{document}

\maketitle
\begin{abstract}
\input{abstract} 
\end{abstract}

% keywords can be removed

\input{introduction}

\input{preliminaries}

\input{problem_statement}

\input{main_results}

\input{applications}

\input{conclusion}

\subsection*{Acknowledgements}

DMT gratefully
acknowledges that the project that gave rise to these results received the support of a fellowship
from ‘la Caixa’ Foundation (ID 100010434). The fellowship code is LCF/BQ/EU22/11930075. AR was funded by NSF grant DMS-2310718.

\bibliographystyle{apalike}
\bibliography{bib}

\appendix

\input{proofs}

\input{related_work}

\input{extended}

\end{document}

%% file: abstract.tex
The study of self-normalized processes plays a crucial role in a wide range of applications, from sequential decision-making to econometrics. While the behavior of self-normalized concentration has been widely investigated for scalar-valued processes, vector-valued processes remain comparatively underexplored, especially outside of the sub-Gaussian framework. In this contribution, we provide concentration bounds for self-normalized processes with light tails beyond sub-Gaussianity (such as Bennett or Bernstein bounds). We illustrate the relevance of our results in the context of online linear regression, with applications in (kernelized) linear bandits.

%% file: introduction.tex
\section{Introduction} \label{section:introduction}

Self-normalized processes naturally arise in a variety of applications, ranging from econometrics \citep{shao2015self, agarwal2018model} and finance \citep{darolles2006structural, pacurar2008autoregressive} to sequential decision making \citep{abbasi2011improved, chowdhury2017kernelized, yang2020function}. 
% In some cases, asymptotic results suffice to yield reliable confidence sets \citep{pena2009self}. However, it may very well be the case that the lack of finite sample guarantees poses a major problem when conducting inference. For example, the probability of failure of a medical sequential decision-making device ought to be understood before deployment with real patients. 
Concentration inequalities for self-normalized processes have been widely studied in the case of scalar-valued random variables \citep{de2004self, de2007pseudo, pena2009self}. Nonetheless, important theoretical challenges arise when working in higher dimensions
% : the absolute value of a one-dimensional process can be studied as the supremum of two processes, thus probabilistic guarantees on the former may be obtained by applying a union bound on the latter two; such an argument proves useless if having infinite directions, which is already the case for a two dimensional process \citep{victor2009theory, whitehouse2023time}. Due to this fact, 
and concentration inequalities for vector-valued self-normalized processes are scarce in the literature, the majority of them assuming sub-Gaussianity of the random variables. If the tails are sub-Gaussian, mixture arguments lead to closed-form, powerful inequalities \citep{abbasi2011improved, chowdhury2017kernelized}. The same methods do not seem easy to generalize to other regimes.

In this work, we aim to establish general concentration inequalities that hold for light tails. More precisely, let us define 
\begin{align} \label{eq:problem_definition}
    V_t = \sum_{i \leq t} X_i X_i^T, \quad M_t = \sum_{i \leq t} X_i \epsilon_i, 
\end{align}
where $X_t$ are sequentially randomly drawn or even adversarially chosen, and $\epsilon_t$ is a real-valued noise with zero expectation and light tails (these definitions will be formalized in Section~\ref{section:problem_statement}).  The ultimate goal of this contribution is to provide new probabilistic guarantees of the form
\begin{align} \label{eq:goal_sequential}
     \| \lp \rho I + V_t \rp^{-\frac{1}{2}} M_t \| \leq J_t(\delta) \text{ simultaneously for all } t \text{ with probability } 1 - \delta 
\end{align}
for different assumptions on the noises $\epsilon_t$. Importantly, note that~\eqref{eq:goal_sequential} establishes a probabilistic guarantee that holds uniformly over time; this sort of guarantee is usually exploited in the fully sequential scenario, where one may peek at the data at any given point of the procedure \citep{ramdas2020admissible, ramdas2023game}. Minimizing $J_t(\delta)$ for a fixed $n = t$ corresponds to the (more classical) batch setting, where the sample size is fixed prior to data collection. In contrast, one may be interested in utilizing \eqref{eq:goal_sequential} without specifying a target sample size in advance, which is useful when developing procedures that may be continuously monitored and adaptively stopped \citep{howard2021time}.

Furthermore, we aim for concentration inequalities that hold in arbitrary Hilbert spaces and so do not have an explicit dependence on the dimension of the Hilbert space. In this regime, much less is known beyond sub-Gaussian $\epsilon_t$. For example,  
can alternative concentration inequalities be provided when $\epsilon_t$ attains the Bernstein's condition, which allows for heavier tails than sub-Gaussian? When $\epsilon_t$ is bounded, can we derive concentration inequalities that adapt to the unknown variance of the random variables, and not only to conservative upper bounds that dictate their sub-Gaussian behavior? This work provides a positive answer to these questions. % \footnote{These questions have also been (partially) addressed by the concurrent contributions \citet{metelli2025generalized}, \citet{akhavan2025bernstein}, and \cite{chugg2025variational}.\tg{this footnote might be useless as the related works are explained right below this sentence}} 

\paragraph{Related work.}  Few contributions have proposed dimension-free self-normalized concentration inequalities for light-tailed noises beyond sub-Gaussianity;  we defer an extended presentation of related work to Appendix~\ref{section:related_work}.  The closest work is that of \citet[Theorem 4.1]{zhou2021nearly}, which presented a Bernstein-type self-normalized concentration inequality for vector-valued martingales, whose proof exploits solely univariate concentration inequalities (similarly to \citet[Theorem 5]{dani2008stochastic}). In contrast to our Bernstein-type inequality, this results in substantially looser constants and an extra logarithmic term in the sample size. Furthermore, our results are significantly more general, leading to Bennett-type inequalities and generalizing to unbounded random variables. 

The recent work of~\cite{whitehouse2023time} does handle (for example) sub-exponential and sub-Poisson noise, but it uses covering arguments to provide dimension-dependent bounds that depend on the condition number of the variance process. In the sub-Gaussian case, their bounds were shown to be incomparable to the better known log-determinant bounds, but the bounds are inherently restricted to finite-dimensional settings, unlike ours.

Concurrent contributions have also explored dimension-independent self-normalized inequalities.  \citet{metelli2025generalized} developed Bernstein-like concentration inequalities for bounded noises relying on stitching arguments, and \citet{akhavan2025bernstein} leveraged method of mixtures and truncation arguments to also develop a Bernstein-like concentration inequality. Again, these two contributions are restricted to the bounded setting, while our bounds are not. 

Finally, the recent preprint of \citet{chugg2025variational} developed inequalities via PAC-Bayes;  while the bounds are very general and explicit, it is not straightforward to analyze the rate because they depend on the inverse of a tail decay function and some ratio of eigenvalues. 

\paragraph{Main contributions and techniques.} Our approach is fundamentally different to the aforementioned works, building on tools originally introduced in~\citet{pinelis1994optimum}. More concretely, we leverage techniques therein to derive a novel nonnegative supermartingale, which leads to clean concentration inequalities when combined with Ville's inequality. Importantly, we characterize the concentration inequalities for the self-normalized processes by clearly decoupling the effect of the directions $(X_t)$ and the tail behavior of the noise $(\epsilon_t)$. In particular, our concentration inequalities are applicable as soon as $\E \lb \exp\lp \lambda |\epsilon_i| \rp -  \lambda |\epsilon_i| - 1\rb$ can be controlled, thus naturally adapting to a larger family of light-tailed noises (e.g., sub-exponential or sub-Poisson), not necessarily sub-Gaussian or even bounded. Furthermore, our concentration inequalities are dimension-free  and applicable in any separable Hilbert space. The inequalities are clean and without large constants, and so they are readily applicable to conduct inference, as we elucidate in Section~\ref{section:applications}.

% The manuscript is organized as follows. Section~\ref{section:related_work} and Section~\ref{section:preliminaries} present related and preliminary work, respectively. The problem under study is later formalized in Section~\ref{section:problem_statement}. Section~\ref{section:main_results} presents the main results of the work, namely a nonnegative supermartingale construction that leads to novel concentration inequalities. Bernstein and Bennett-type inequalities are developed, with the latter being extended to the fully sequential setting, as well as an empirical version. Section~\ref{section:applications} exhibits the immediate implications of our work in the linear bandit setting. We lastly conclude with some remarks in Section~\ref{section:conclusion}.  

%% file: preliminaries.tex
\section{Preliminaries} \label{section:preliminaries}

\subsection{Separable Hilbert spaces}

Throughout, we will work with random elements in a separable Hilbert space $H$. We remind the reader that a Hilbert space is a complete inner product space. A \textit{separable} Hilbert space is a Hilbert space that contains a countable, dense subset.  Any separable Hilbert space is linearly isometric to $l^2(\Nb) = \{ (x_i)_{i \in \Nb}: \sum_{i \in \Nb}x_i^2 < \infty \}$, so we can think of separable Hilbert spaces as potentially infinite sequences whose sum of squares is finite (thus generalizing the usual, finite dimensional Euclidean spaces). Separability is usually assumed given its generality and measure-theoretic convenience (among other reasons, see e.g. \citet[Chapter 2]{ledoux2013probability}). 
We highlight that every Hilbert space is a $(2, 1)$-smooth Banach space \citep{pinelis1994optimum}. 

\paragraph{Notation.} Given two elements $f, g \in H$, we denote their inner product by either $\la f, g \ra$ or $f^T g$, and their outer product by $fg^T := f \la g, \cdot \ra$, which is a linear operator from $H$ to itself. Similarly, the identity operator from $H$ to itself is denoted by $I$.

\subsection{Nonnegative supermartingales and Ville's inequality}

%Let us start by presenting some essential tools. 
Nonnegative supermartingales play a central role in deriving anytime valid concentration inequalities due to Ville's inequality \citep{ville1939etude}. Before presenting the inequality, we introduce some notation. Let $\mathbb{N} = \{0,1,2,...\}$. A  filtration $\F = (\F_t)_{t \in \mathbb{N}}$ is a sequence of $\sigma$-algebras such that $\F_t \subseteq \F_{t+1}$, for all $t$. A stochastic process $M = (M_t)_{t \in \mathbb{N}}$ is a sequence of random variables that are adapted to $(\F_t)_{t \in \mathbb{N}}$, meaning that $M_t$ is $\F_{t}$-measurable for all $t$. $M$ is called \emph{predictable} if $M_t$ is $\F_{t-1}$-measurable for all $t$. 
% Intuitively, this means that $P_t$ only depends on $X_1, \ldots, X_{t}$ and $X_1, \ldots, X_{t-1}$, respectively. 
An integrable stochastic process $M$ is a supermartingale with respect to $\F$ if $\E[M_{t+1}| \F_t] \leq M_t$  for all $t$, and a martingale if the inequality always holds with equality;  inequalities or equalities between random variables are always interpreted to hold almost surely. Throughout, we use the shorthand $\E_{t}[\cdot]$ for $\E[\cdot | \F_t]$.  %We are now ready to present Ville's inequality. 
\begin{fact} [Ville's inequality] \label{fact:villesineq} If $M$ is a nonnegative supermartingale (with respect to any filtration $\F$), then for any $x > 0$,
\begin{align*}
    \Pb(\exists t \in \mathbb{N} : M_t \geq x) \leq \frac{\E[M_0]}{x}. 
\end{align*}
\end{fact}
This result can be seen as a time-uniform version of Markov's inequality, and yields anytime-valid concentration inequalities in a similar way in which Markov's does for fixed $t$. Indeed, note that if $M_0 = 1$, then selecting $x = 1/\delta$ in Ville's inequality implies that  
\[M_t \leq 1/\delta \text{ for all } t \in \mathbb{N} \text{ with probability at least } 1 - \delta,\]
mirroring the form of the probabilistic guarantee stated in \eqref{eq:goal_sequential}. To obtain the exact expression in \eqref{eq:goal_sequential}, we will need to carefully rearrange the inequality. The main technical challenge then is to appropriately construct (super)martingales and rearrange the obtained inequalities.

\subsection{Vector-valued concentration beyond subGaussianity}

The nonnegative supermartingale derived in this contribution makes essential use of the vector-valued results from~\citet{pinelis1994optimum}. In contrast to most prominent approaches regarding vector-valued concentration--which exploit either covering arguments, or PAC-Bayes, or mixture martingales--\citet{pinelis1994optimum} directly worked with the properties of the norm. In particular, we use in this contribution the following result, which is a simplified version of \citet[Theorem 3.2]{pinelis1994optimum}.

\begin{fact} \label{fact:pinelis} Let $f \in H$ be fixed, and $X \in H$ be a random element. It holds that  
    \begin{align*}
    \E  \cosh  \ldba f + \lambda X \rdba   &\leq \lp 1 + \tilde e_t(\lambda) \rp \cosh \ldba f \rdba ,
    \end{align*}
    where $\tilde e_t(\lambda) := \E \lb \exp\lp \lambda \ldba X \rdba \rp -  \lambda \ldba X \rdba - 1\rb.$
\end{fact}

Note that Theorem~\ref{fact:pinelis} allows for working with the moment generating function of $\|X\|$ but without the first order term, which bypasses the main obstacle when using other well-known techniques that relate to Chernoff-bounds. Nonetheless, there is no sense of self-normalization in the original formulation of this inequality, which is the main object under study in this contribution.

%% file: problem_statement.tex
\section{Problem statement} \label{section:problem_statement}

Let $(X_t)_{t\geq0}$ be a $H$-valued stochastic process, where $H$ is a separable Hilbert space, and let $(\epsilon_t)_{t\geq0}$ be a real-valued stochastic process. We are interested in deriving time-uniform concentration inequalities for $\| \lp \rho I + V_t \rp^{-\frac{1}{2}} M_t \|$, where $$V_t = \sum_{i \leq t} X_i X_i^T, ~ M_t = \sum_{i \leq t} X_i \epsilon_i$$ and $\rho>0$ is fixed. Throughout, we work under the following assumption. 

\begin{assumption} \label{assumption:adapted}
    There exists a filtration $\F = (\F_t)_{t \geq 0}$ such that (i) $X_t$ is $\F_{t-1}$-measurable for all $t$, (ii) $\epsilon_t$ is $\F_{t}$-measurable for all $t$, and (iii) $\E_{t-1} \epsilon_t = 0$, where we recall that $\E_{t-1}[\cdot] = \E[\cdot | \F_{t-1}]$.
\end{assumption}
In many applications, $\F$ is naturally taken as the canonical filtration generated by both stochastic processes. That is, $\F_t = \sigma((X_1, \epsilon_1), \ldots, (X_t,\epsilon_t))$, with $\F_0$ being trivial. In this case, (i) implies that $X_t$ is drawn from a distribution that can depend on $(X_1, \epsilon_1, X_2, \epsilon_2, \ldots, X_{t-1}, \epsilon_{t-1})$, but cannot depend on the outcomes of (yet to be seen) later rounds. This setting encompasses both independent and identically distributed (iid) draws for $X_t$, as well as potentially `adversarially' chosen $(X_t)$, where the adversary has the information only until round $t-1$. The noises $\epsilon_t$ can further depend on $X_t$, thereby accommodating heteroskedasticity, which corresponds to (ii). However, (iii) imposes a martingale structure on $(\epsilon_t)$: intuitively, we can think of $\epsilon_t$ as noise and thus it has zero conditional expectation. As we shall see in later sections, working under Assumption \ref{assumption:adapted} allows for providing concentration inequalities for the self-normalized processes. In particular, the tails of the noises $\epsilon_t$ play a key role in such concentration bounds. We focus on two types of light-tailed noises, given their importance and ubiquity in the literature. 

\begin{assumption} [Bernstein condition]  \label{assumption:bernstein_condition}
    There exist $B$ and $(\sigma_t)$ such that $(\epsilon_t)$ fulfils
    \begin{align*}
        \E_{t-1} |\epsilon_t|^m \leq \frac{1}{2} m! B^{m-2} \sigma_t^2 \quad \forall m \geq 2, \forall t\geq 1,
    \end{align*}
    where $\sigma_t$ is $\F_{t-1}$-measurable.
\end{assumption}

Assumption \ref{assumption:bernstein_condition} is commonly summarized as: $\epsilon_t$ satisfies the Bernstein condition with parameters $(\sigma_t, B)$. A sufficient condition for Bernstein's condition to hold is that the random variable is bounded, but it can also be satisfied by several unbounded variables (in fact, it is equivalent to a sub-exponential tail bound condition, see e.g. \citet{howard2020time}), giving it much wider applicability.

\begin{assumption}  [Bounded noise and variance]\label{assumption:bounded_variance}
    There exist $B$ and $(\sigma_t)$ such that the stochastic process $(\epsilon_t)$ fulfils, for all $t \geq 0$,
    \begin{align*}
        |\epsilon_t| \leq B, \quad \V_{t-1}[\epsilon_t] \leq \sigma_t^2,
    \end{align*}
    where $\sigma_t$ is $\F_{t-1}$-measurable.
\end{assumption}

Assumption \ref{assumption:bounded_variance} imposes  boundedness on the random noises (and hence their variance). Although bounded random variables are also sub-Gaussian, exploiting (the upper bound of) the variance can lead to sharper results.

While we focus our efforts in random noises that fulfill one of the former two assumptions for clarity of exposition, the results presented in this work hold as soon as the conditional expectations $\E_{t-1} [\exp(\lambda |\epsilon_t|) - \lambda |\epsilon_t| - 1]$ can be controlled for (with the former assumptions being specific instances of this case). Thus, the applicability of this work also goes beyond these two examples.

%% file: main_results.tex
\section{Main results} \label{section:main_results}

We present in this section the main results of the contribution, namely a supermartingale construction alongside the concentration inequalities that can be derived from it. Contrary to current approaches, our supermartingale construction cleanly decouples the norm representation of the vector-valued process from the concentration of the one-dimensional noises. In order to do that, we elucidate in Section \ref{section:gt} the behaviour of the terms $\|G_t\|$, where
\begin{align*}
     G_t :=  \bigg(\rho I + \sum_{i \leq t} X_i X_i^T\bigg)^{-\frac{1}{2}} X_t.
\end{align*}
Section \ref{section:main_theorem} combines such $G_t$ terms with the noises in order to provide a supermartingale construction for our central object of study.  Such a supermartingale construction composes the main theoretical contribution of this work, from which the remaining results can be derived. The pseudo-variance process $\sum_{i \leq t} \|G_i\|^2$ plays an important role in the supermartingale construction, and  we devote Section~\ref{section:upper_bound_variance} to elucidate its connection to the intrinsic dimension of the Hilbert space. Concentration inequalities are then obtained in Section~\ref{section:corollaries}, where inequalities for both the fixed time and fully sequential settings are rigorously derived.  In particular, Bernstein, Bennett, and empirical Bennett-type concentration inequalities are presented.  %We defer all the proofs to Appendix~\ref{section:proofs}.

\subsection{Decoupling the direction from the noise} \label{section:gt}

We highlight the decomposition of the self-normalized process
\begin{align} \label{eq:decomposition_ut}
    \lp \rho I + V_t \rp^{-\frac{1}{2}} M_{t-1} + \lp \rho I + V_t \rp^{-\frac{1}{2}} X_t\epsilon_t.
\end{align}
As it shall be seen in Section \ref{section:main_theorem}, controlling the tails of the second term in \eqref{eq:decomposition_ut}  suffices to provide a supermartingale construction, which leads to concentration bounds $J_t(\delta)$. We observe that 
\begin{align*}
    \ldba \lp \rho I + V_t \rp^{-\frac{1}{2}} X_t \epsilon_t \rdba = \ldba \lp \rho I + V_t \rp^{-\frac{1}{2}} X_t  \rdba \lba \epsilon_t \rba = \ldba G_t \rdba  \lba \epsilon_t \rba,
\end{align*}
which decouples the effect of the noise $\epsilon_t$ from the vector $G_t$ on which it is projected. In particular, $\|G_t\|$ depends on $\rho$, thereby concealing the effect of $\rho$ on the bound. It is also key to note that $\ldba G_t \rdba \leq 1$. Indeed, we observe that
\begin{align*}
     (\rho I + \sum_{i \leq t} X_i X_i^T) \succeq (\rho I + X_t X_t^T),
\end{align*}
and so, by using the Sherman-Morrison rank-one update formula,
\begin{align*}
    \| G_t \|^2 \leq \ldba  (\rho I + X_t X_t^T)^{-1/2} X_t \rdba^2 = X_t^T(\rho I + X_t X_t^T)^{-1} X_t = \frac{\|X_t\|^2}{\rho + \|X_t\|^2} \leq 1.
\end{align*}

\subsection{Light-tailed self-normalized process supermartingale constructions} \label{section:main_theorem}

% We are now ready to present the main results of this contribution. The inequalities constructed in this paper are derived in two steps. First, nonnegative supermartingale constructions are presented. Second, concentration inequalities are derived and tuned in view of such constructions. 

%Throughout, we denote 
% \begin{align*}
%      M_t^\lambda :=  \lambda \sum_{i \leq t} X_i \epsilon_i.
% \end{align*}

We now provide a nonnegative supermartingale that adapts to different tail behaviours of the random noises through $\E_{t-1}\lb \exp \lp \lambda  \lba \epsilon_t \rba\rp- \lambda  \lba \epsilon_t \rba - 1\rb$. Hence, this construction may be exploited as long as the moment generating function of the absolute value of the noises can be controlled for. As it is usual when deriving concentration inequalities, we will work with the process that is object of study multiplied by an arbitrary value $\lambda > 0$, that is then optimally adjusted or mixed to yield tight inequalities. The next theorem establishes such a nonnegative supermartingale construction.

\begin{theorem} \label{theorem:main_theorem}
%Let $(X_t)_{t \geq 1}$ be a predictable Hilbert space valued process, and $(\epsilon_t)_{t \geq 1}$ an adapted, real valued sub-$\psi$ with variance process $\sigma_t$, both with respect to filtration $\F = (\F_t)_{t \geq 0}$. 
Let $(X_t)_{t \geq 1}$ and $(\epsilon_t)_{t \geq 1}$ be Hilbert space valued and real valued processes, respectively, attaining Assumption~\ref{assumption:adapted}.
Let $\lambda>0$ and recall that $G_t = \lp \rho I + V_t \rp^{-\frac{1}{2}} X_t$. Denoting 
\begin{align*}
    e_t(\lambda) =  \ldba G_t\rdba^2\E_{t-1}\lb \exp \lp \lambda  \lba \epsilon_t \rba\rp- \lambda  \lba \epsilon_t \rba - 1\rb, %V_t = \sum_{i \leq t} X_i X_i^T, \quad G_t = (\rho I + V_t)^{-1/2} X_t, \quad
\end{align*}
the process
\begin{align} \label{eq:main_process} 
        S_t = \cosh \lp \lambda \ldba (\rho I + V_t)^{-1/2}M_t  \rdba \rp \exp \lp - \sum_{i \leq t} e_i(\lambda)    \rp
    \end{align}
is a nonnegative supermartingale.
\end{theorem}

\begin{proof}
    To prove that 
% \[S_t = \cosh \lp \lambda \ldba (\rho I + V_t)^{-1/2}M_t \rdba \rp \exp \lp - \sum_{i \leq t} e_i(\lambda)    \rp\]
$S_t$ is a supermartingale (it is trivially nonnegative), we first observe that 
    \begin{align*}
        \ldba \lambda (\rho I + V_t)^{-1/2}M_t \rdba &= \ldba \lambda (\rho I + V_t)^{-1/2}M_{t-1} + (\rho I + V_t)^{-1/2} \lp \lambda X_t \epsilon_t\rp \rdba
        \\& = \ldba \lambda (\rho I + V_t)^{-1/2}M_{t-1} +   \lambda G_t \epsilon_t \rdba.
    \end{align*}
Plugging the above into Fact~\ref{fact:pinelis} establishes that for  $\tilde e_t(\lambda) := \E_{t-1} \lb \exp\lp \lambda \ldba G_t \epsilon_t \rdba \rp -  \lambda \ldba G_t \epsilon_t \rdba - 1\rb$, we have
\begin{align*}
    \E_{t-1} \lb \cosh \ldba\lambda  (\rho I + V_t)^{-1/2}M_t \rdba \rb &\leq \lp 1 + \tilde e_t(\lambda) \rp \cosh \ldba \lambda (\rho I + V_t)^{-1/2}M_{t-1} \rdba
    \\&\leq \exp \lp \tilde e_t(\lambda) \rp \cosh \ldba\lambda (\rho I + V_t)^{-1/2}M_{t-1} \rdba.
\end{align*}
Since $\ldba G_t \rdba \leq 1$ and it is $\F_{t-1}$-measurable, we see that
\begin{align*}
    \tilde e_t(\lambda) &= \E_{t-1} \lb \exp\lp \lambda \ldba G_t \epsilon_t \rdba \rp -  \lambda \ldba G_t \epsilon_t \rdba - 1\rb \\
    %\\&= \E_{t-1}\lb \sum_{k \geq 2} \frac{\lp \lambda \ldba G_t \epsilon_t \rdba\rp^k}{k!}\rb
&= \E_{t-1}\lb \sum_{k \geq 2} \frac{\lp \lambda \ldba G_t\rdba \lba \epsilon_t \rba\rp^k}{k!}\rb
    \\&\stackrel{}{\leq} \ldba G_t\rdba^2\E_{t-1}\lb \sum_{k \geq 2} \frac{\lp \lambda  \lba \epsilon_t \rba\rp^k}{k!}\rb \stackrel{}{=} e_t(\lambda).
\end{align*}
 Plugging this back into the earlier expression, and noting that $ V_{t-1} \preceq V_t $, we have
\begin{align*}
      \E_{t-1} \lb  \cosh \ldba \lambda  (\rho I + V_t)^{-1/2}M_t \rdba \rb&\leq   \exp \lp e_t(\lambda) \rp \cosh \ldba \lambda  (\rho I + V_t)^{-1/2}M_{t-1} \rdba
      \\&\stackrel{}{\leq} \exp \lp e_t(\lambda) \rp \cosh \ldba \lambda (\rho I + V_{t-1})^{-1/2}M_{t-1} \rdba.
\end{align*}
This implies the supermartingale property:
\begin{align*}
    \E_{t-1}[S_t] &= \E_{t-1}\left[\cosh \lp  \ldba \lambda (\rho I + V_t)^{-1/2}M_t  \rdba \rp \exp \lp - \sum_{i \leq t} e_i(\lambda)    \rp\right]\\
    &\leq \exp \lp e_t(\lambda) \rp \cosh \ldba \lambda (\rho I + V_{t-1})^{-1/2}M_{t-1} \rdba \exp \lp - \sum_{i \leq t} e_i(\lambda)    \rp\\
    &= S_{t-1}.
\end{align*}
%succeq
\end{proof}

Intuitively, the supermartingale is constructed by re-expressing the normalized update so that all geometric complexity is absorbed into $\|G_t\|$, while all randomness is carried solely by the scalar noise $|\epsilon_t|$. In this ``whitened'' geometry, the vector-valued process behaves like a one-dimensional process whose increments have size at most $|\epsilon_t|$. Applying Pinelis’ inequality to the $\cosh$ of the normalized radius yields a predictable exponential correction term, and subtracting these terms creates a supermartingale.

Note that if the noises $(\epsilon_t)$ attain the Bernstein condition (Assumption \ref{assumption:bernstein_condition}) with parameters $(\sigma_t, B)$, then (see e.g. \citet[Equation 2.16]{wainwright2019high}) 
\begin{align*}
    \E_{t-1}\lb \exp \lp \lambda  \lba \epsilon_t \rba\rp- \lambda  \lba \epsilon_t \rba - 1\rb \leq \frac{\lambda^2}{2(1 - \lambda B)}\sigma_t^2 , \quad 0 \leq \lambda < \frac{1}{B}.
\end{align*}
Otherwise, if the noises $(\epsilon_t)$  are bounded by $B$ and their variance is $(\sigma_t^2)$ (Assumption \ref{assumption:bounded_variance}), it follows that
\begin{align*}
    \E_{t-1}\lb \exp \lp \lambda  \lba \epsilon_t \rba\rp- \lambda  \lba \epsilon_t \rba - 1\rb \leq  \frac{e^{\lambda B} - \lambda B - 1}{B^2}\sigma_t^2.
\end{align*}
The shrinkage factors $\exp ( - \sum_{i \leq t} e_i(\lambda) )$ that emerge in our nonnegative supermartingale may appear too large at first glance, given its non-standard form. Shrinkage factors are usually found in the form $\exp ( - \psi(\lambda) \sum_{i \leq t} \varsigma_i^2 )$, where $\psi$ is a CGF-like function, and $\varsigma_i^2$ are conditional pseudo-variances; see \citet{howard2020time} for a discussion on this general sub-$\psi$ framework. Despite their different appearances, these factors are analogous. For example, the usual Bennett-inequality includes the shrinkage factor  $\exp ( - \psi_{P, B}(\lambda) \sum_{i \leq t} \varsigma_i^2 )$, where $\psi_{P, B}(\lambda) := (\exp(\lambda B) - \lambda B - 1) / B^2$ is the so-called sub-Poisson $\psi$-function. In contrast, our upper bound for $\E_{t-1} [ \exp (\lambda  | \epsilon_t | )- \lambda | \epsilon_t | - 1]$ for bounded noises  is precisely $\psi_{P, B}(\lambda) \sigma_t^2$. Thus, we recover the sub-$\psi_{P, B}$ shrinkage factor with $\varsigma_i^2 = \|G_i\|^2\sigma_i^2$. Similarly, the Bernstein condition upper bound  leads to a shrinkage factor of $\exp ( - \psi_{G, B}(\lambda) \sum_{i \leq t} \|G_i\|^2\sigma_i^2 )$, where $\psi_{G, B} = \lambda^2 / (2 - 2B\lambda)$ is the usual sub-Gamma  $\psi$-function. Thus, we see that our results fall under the umbrella of sub-$\psi$ shrinkage factors. 

Although we generally recover sub-$\psi$ shrinkage factors, we do not do so via the usual sub-$\psi$ definition. That is, we do not assume that $\E_{t-1}\exp(\lambda \epsilon_t) \leq \exp(\psi(\lambda) \varsigma_t^2)$. Instead, we (implicitly) exploit an upper bound of the form $\E_{t-1} [ \exp (\lambda  | \epsilon_t | )- \lambda | \epsilon_t | - 1] \leq \psi(\lambda)\varsigma_t^2$. Equivalently, we can rewrite this inequality as $\E_{t-1} [ \exp (\lambda  | \epsilon_t | )- \lambda | \epsilon_t |] \leq 1 + \psi(\lambda)\varsigma_t^2$. Now we note that 
\begin{align*}
    1 + \psi(\lambda)\varsigma_t^2 \leq \exp(\psi(\lambda) \varsigma_t^2),
\end{align*} 
and if $\E_{t-1}\epsilon_t$ = 0, the inequality 
\begin{align*}
    \E_{t-1}\exp(\lambda \epsilon_t) = 1 + \E_{t-1}\sum_{k \geq2} \frac{(\lambda \epsilon_t)^k}{k!}\leq 1 + \E_{t-1}\sum_{k \geq2} \frac{(\lambda |\epsilon_t|)^k}{k!}\leq \E_{t-1} [ \exp (\lambda  | \epsilon_t | )- \lambda | \epsilon_t |]
\end{align*}
holds as well. This implies that our conditition on $\epsilon_t$ is slightly stronger than the usual sub-$\psi$ condition, and so we cannot present our results under the usual sub-$\psi$ framework. Nonetheless, many of the predominant random variables that are sub-$\psi$ also attain our (only slightly) stronger condition. Clear cases are the aforementioned bounded and Bernstein condition noises achieving our condition, on top of being sub-Poisson and sub-Gamma respectively.

Theorem~\ref{theorem:main_theorem} allows for obtaining concentration inequalities in view of Ville's inequality. The result is formalized in the following proposition, and its proof can be found in Appendix~\ref{proof:main_cor}.

\begin{proposition}[Light-tailed self-normalized process concentration inequality] \label{corollary:main_cor}
Under the assumptions of Theorem \ref{theorem:main_theorem}, it holds that, with probability $1 - \delta$ and simultaneously for all $t \geq 1$,
\begin{align*}
    \ldba (\rho I + V_t)^{-1/2}M_t \rdba  \leq \frac{ \sum_{i=1}^t e_i(\lambda) + \log \lp\frac{2}{\delta}\rp}{\lambda}.
\end{align*}
\end{proposition}
Proposition~\ref{corollary:main_cor} immediately yields concentration inequalities for our central object of interest. Nonetheless, the bound provided by Proposition~\ref{corollary:main_cor} depends on an arbitrary value $\lambda$ that has to be chosen prior to data collection. Thus,  the tightness of the concentration bounds entirely relies on the choice for $\lambda$. Furthermore, if the term $\E_{t-1}\lb \exp \lp \lambda  \lba \epsilon_t \rba\rp- \lambda  \lba \epsilon_t \rba - 1\rb$ is constant across $t$, the sum of $e_i(\lambda)$ scales with $\sum_{i \leq t} \|G_i\|^2 \sigma_i^2$, and the optimal choice of $\lambda$ relies on an upper bound of this process.  %which we call the `variance process' (given that it captures the potential variability of the process over time). The optimal choice of $\lambda$ relies on (an upper bound of) this `variance process'. 

While our focus is on non-asymptotic concentration inequalities, Theorem~\ref{theorem:main_theorem} also leads to the so-called laws of the iterated logarithm \citep{kolmogoroff1929gesetz, stout1970martingale}. The laws of the iterated logarithm characterize the almost-sure limiting envelope of normalized sums of random variables, describing the oscillatory behavior of stochastic processes between the law of large numbers and the central limit theorem. In particular, the following general upper asymptotic law of the iterated logarithm is immediately obtained from Theorem~\ref{theorem:main_theorem} in combination with \citet[Corollary 1]{howard2021time}; its proof is presented in Appendix~\ref{proof:lil}. One can also use \citet[Theorem 1]{howard2021time} to derive explicit nonasymptotic laws of the iterated logarithm bounds directly from Theorem~\ref{theorem:main_theorem}.

%\footnote{We defer the reader to \citet[Section 3.1]{howard2021time} for a discussion on how to also obtain non-asymptotic laws of the iterated logarithm from nonnegative supermartingales (such as Theorem~\ref{theorem:main_theorem}) via stitching arguments.}

\begin{corollary} \label{corollary:lil}
Under Assumption~\ref{assumption:adapted}, and either Assumption~\ref{assumption:bernstein_condition} or Assumption~\ref{assumption:bounded_variance}, 
\begin{align*}
    \limsup_{t \to \infty} \frac{\ldba (\rho I + V_t)^{-1/2}M_t \rdba}{\sqrt{2 \lp\sum_{i \leq t} \|G_i\|^2 \sigma_i^2 \rp \log\log\lp\sum_{i \leq t} \|G_i\|^2 \sigma_i^2 \rp}} \leq 1 \quad \text{on } \lc \sup_t{\sum_{i \leq t} \|G_i\|^2 \sigma_i^2} = \infty\rc.
\end{align*}
    
\end{corollary}

\subsection{Upper bounding the pseudo-variance process} \label{section:upper_bound_variance}

While we have stated all our results for arbitrary separable Hilbert spaces, prominent examples of Hilbert spaces in statistical applications are finite-dimensional Euclidean spaces and reproducing kernel Hilbert spaces (RKHS). Given that finite-dimensional Euclidean spaces are instances of RKHS's, our discussion will be framed in terms of the latter.

Often, the quantity $\|G_t\|$ is expected to shrink as $t$ grows, because the matrix $V_t$ accumulates information and enlarges the ellipsoid in which the normalization occurs. As more covariates are observed, the directions in which the data has significant energy get scaled down more aggressively by $(\rho I + V_t)^{-1/2}$, making each $X_t$ appear smaller in the normalized geometry. Intuitively, new observations rarely point in directions that have not already been heavily regularized. This progressive flattening of the geometry is what causes $\|G_t\|$ to typically decrease.

More specifically, note that $\|G_i\| \leq 1$ and $x \leq 2\log(1 + x)$ for $x\in[0,1]$, so
\begin{align*}
    \sum_{i \leq t} \|G_i\|^2 &\leq 2 \sum_{i \leq t} \log \lp 1 + \|G_i\|^2\rp = 2 \sum_{i \leq t} \log \lp 1 + \ldba \lp \rho I + V_t \rp^{-\frac{1}{2}} X_t  \rdba^2 \rp 
    \\&\leq 2 \sum_{i \leq t} \log \lp 1 + \ldba \lp \rho I + V_{t-1} \rp^{-\frac{1}{2}} X_t  \rdba^2 \rp = 2\log\det \lp I + \rho^{-1} V_t\rp,
\end{align*}
where the last inequality follows from $V_{t-1} \preceq V_t$, and the last equality from the elliptical potential lemma
\citep[proof of Lemma~11]{abbasi2011improved}. For bounded $(X_t)$, one may further consider the upper bound 
\begin{align} \label{eq:maximal_information_gain}
     \frac{1}{4}\sum_{i \leq t} \|G_i\|^2 \leq \sup_{V_t} \frac{1}{2} \log \det \lp I + \rho^{-1} V_t\rp =: \gamma_t(\rho),
\end{align}
with $\gamma_t(\rho)$ being the maximal information gain, a concept that relates to the intrinsic dimension of the RKHS and that has been widely exploited in sequential decision-making problems (such as bandits). Consequently, upper bounds for the RHS of \eqref{eq:maximal_information_gain} have already been established for bounded $(X_t)$, e.g. in \citet[Corollary 1]{vakili2021information}; these have no explicit dimension dependence and are thus called ``dimension-free'' (so they are applicable in infinite dimensional settings, where they will capture some notion of effective dimensionality). For Mercer kernels, such bounds depend on their Mercer eigenvalue decay. Of special interest are the scenarios where most of the eigenvalues $(\lambda_m)$ are $0$ (finite dimensional Euclidean spaces), or decay exponentially or polynomially fast (generally RBF and Laplace kernels, respectively).

\subsection{Concentration inequalities} \label{section:corollaries}

Proposition~\ref{corollary:main_cor} establishes a very general concentration inequality, whose tightness depends on the form of $e_i(\lambda)$ and the choice of $\lambda$. In particular, the previously introduced $e_i(\lambda)$ will yield Bernstein-type and Bennett-type inequalities. Furthermore, the choice of $\lambda$ is also motivated by the nature of the concentration inequalities, with different choices for the fixed sample size and the fully sequential settings. We provide specific instances of such inequalities. We start with a fixed sample size Bernstein-type inequality. Its proof, which optimizes for $\lambda$ in Proposition~\ref{corollary:main_cor} for the Bernstein-specific $e_i(\lambda)$, is deferred to Appendix~\ref{proof:fixed_bernstein}.%, with emphasis in both the batch and sequential settings. Sometimes, we may assume that an upper bound for $\sum_{i \leq t} \|G_i\|^2$ is known, which is reasonable in prominent examples as discussed in the previous section. 

\begin{theorem}[Bernstein-type concentration inequality] \label{corollary:fixed_bernstein}
    Fix $n > 0$, and let $(X_1, \ldots, X_n)$ and $(\epsilon_1, \ldots, \epsilon_n)$ fulfill Assumption \ref{assumption:adapted}. Let $\epsilon_i$ attain the Bernstein condition with parameters $(\sigma_i, B)$ (Assumption \ref{assumption:bernstein_condition}). If $\sum_{i \leq n} \sigma_i^2 \| G_i \|^2 \leq C_n^2$ almost surely, where $C_n \geq 0$ is deterministic, then
    \begin{align*}
        \sup_{t \leq n}\ldba (\rho I + V_t)^{-1/2}M_t \rdba  &\leq B \log \lp \frac{2}{\delta} \rp + C_n\sqrt{2  \log \lp \frac{2}{\delta}\rp}
    \end{align*}
    with probability $1-\delta$. Equivalently, the following holds for all $r\ge0$:
    \begin{align*}
        \Pb \lp \sup_{t\leq n} \|(\rho I+V_t)^{-1/2}M_t\| \geq r \rp \leq 2 \exp \lp - \frac{r^2}{2\lp  C_n^2 + Br \rp } \rp.
    \end{align*}
\end{theorem}
% The theorem could have also been stated in its tail-probability form as 
% \begin{align*}
%         \Pb \lp \sup_{1 \leq i \leq n} \|(\rho I+V_t)^{-1/2}M_t\| \geq r \rp \leq 2 \exp \lp - \frac{r^2}{2\lp  C_n^2 + Br \rp } \rp
%     \end{align*}
% for all $ r \geq 0$.

If the noises $(\epsilon_i)$ are bounded and (an upper bound on) their variance is known, a Bennett-type inequality will be tighter than a Bernstein-type inequality. The proof of the following result, which can be derived similarly to that of the Bernstein-type inequality, is given in Appendix~\ref{proof:fixed_bennett}.

\begin{theorem}[Bennett-type concentration inequality] \label{corollary:fixed_bennett}
    Fix $n > 0$, and let $(X_1, \ldots, X_n)$ and $(\epsilon_1, \ldots, \epsilon_n)$ attain Assumption \ref{assumption:adapted}. Let $|\epsilon_i| \leq B$ and $\V_{i-1} [\epsilon_i] \leq \sigma^2_i$ almost surely (Assumption~\ref{assumption:bounded_variance}). If $\sum_{i \leq n} \sigma_i^2 \| G_i \|^2 \leq C_n^2$ almost surely, where $C_n \geq 0$ is deterministic, then 
    \begin{align*}
        \sup_{t \leq n}\ldba (\rho I+V_t)^{-1/2}M_t\rdba
    \leq \frac{C_n^2}{B} h^{-1}\lp\frac{B^2}{C_n^2}\log\lp\frac{2}{\delta}\rp\rp
    \end{align*}
    with probability at least $1 - \delta$, where $h(u) = (1 + u) \log (1 + u) - u$. Equivalently, the following holds for all $r\ge0$:
    \begin{align*}
    \Pb \lp \sup_{t \leq n} \ldba (\rho I+V_t)^{-1/2}M_t\rdba \geq r \rp \leq 2 \exp \lp - \frac{C_n^2}{B^2} h \lp \frac{Br}{C_n^2}\rp \rp.
    \end{align*}
\end{theorem}

% While we have stated Theorem~\ref{corollary:fixed_bennett} as a function of $\delta$, it can be rewritten in the tail-probability form
% \begin{align*}
%     \Pb \lp \sup_{1 \leq i \leq n} \ldba (\rho I+V_t)^{-1/2}M_t\rdba \geq r \rp \leq 2 \exp \lp - \frac{C_n^2}{B^2} h \lp \frac{Br}{C_n^2}\rp \rp,
% \end{align*}
% holding for all $r\geq0$. %As it is usual with Bennett-type inequality, the concentration inequality from Corollary~\ref{corollary:fixed_bennett} cannot be obtained in closed-form. A loose bound on the function $h$ \citep{mcdiarmid1998concentration} leads to the following closed-form confidence sequence, which is only a slight improvement over Bernstein's inequality.
% \begin{align*}
%     \Pb \lp \sup_{1 \leq i \leq n} \|J_i\| \leq \sigma C_n\sqrt{2\log \lp 2/\delta\rp} + \frac{B\log \lp 2/\delta\rp}{3}\rp \geq 1- \delta. 
% \end{align*}

We highlight that both Theorem~\ref{corollary:fixed_bernstein} and Theorem~\ref{corollary:fixed_bennett} exploit deterministic upper bounds $C_n$, which require a fixed sample size. These concentration inequalities usually suffice when used to conduct theoretical analyses. However, we should expect these inequalities to be conservative in practice due to their use of a deterministic upper bound $C_n$. Furthermore, they do not easily adapt to the fully sequential setting, where sample sizes are random stopping times. For bounded random noises, and at an expense of a logarithmic term, these two shortcomings can be addressed by means of a mixture method argument.

\begin{theorem} [Mixed Bennett-type concentration inequality] \label{corollary:mixed_bennett}
    Let $(X_1, X_2, \ldots)$ and $(\epsilon_1, \epsilon_2, \ldots)$ attain Assumption \ref{assumption:adapted}, and let $\epsilon_i$ attain Assumption~\ref{assumption:bounded_variance}, i.e., $|\epsilon_i| \leq B$ and $\V_{i-1} [\epsilon_i] \leq \sigma^2_i$. Denote
    \begin{align*}
        s_t = \| \lp \rho I + V_t \rp^{-\frac{1}{2}} M_t \|, \quad \nu_t = \sum_{i \leq t} \sigma_i^2 \| G_i \|^2.
    \end{align*}
    For $\theta > 0$, it holds that
    \begin{align} \label{eq:mixed_bennett_eq}
        \Pb \lp \sup_{t} \frac{\lp \frac{\theta}{B^2} \rp^\frac{\theta}{B^2}}{\Gamma \lp \frac{\theta}{B^2}\rp \overline{\gamma} \lp \frac{\theta}{B^2}, \frac{\theta}{B^2}\rp} \frac{\Gamma \lp  \frac{Bs_t + \nu_t + \theta}{B^2} \rp \overline{\gamma} \lp \frac{Bs_t + \nu_t + \theta}{B^2}, \frac{\nu_t + \theta}{B^2} \rp}{\lp \frac{\nu_t + \theta}{B^2}\rp^{\frac{Bs_t + \nu_t + \theta}{B^2}} } \exp\lp{\frac{\nu_t}{B^2}}\rp < \frac{2}{\delta} \rp \geq 1 - \delta,
    \end{align}
    where $\overline{\gamma}(a, x) := (\int_x^\infty u^{a-1}e^{-u}du)/\Gamma(a)$ is the regularized upper incomplete gamma function.
\end{theorem}
Its proof is deferred to Appendix~\ref{proof:mixed_bennett}, and combines Theorem~\ref{theorem:main_theorem} with the Gamma-Poisson mixture argument from \citet[Proposition 10]{howard2021time}. Theorem~\ref{corollary:mixed_bennett} does not offer a closed-form confidence interval, but it can be easily obtained by root finding \citep[Appendix D]{howard2021time}. The hyperparameter $\theta$ can be adjusted by following the considerations from~\citet[Section 3.5]{howard2021time}. While the probabilistic expression is rather difficult to parse, the width of the confidence interval can be studied using the tools from~\citet{howard2021time}. In particular,~\citet[Proposition 2]{howard2021time} establishes that the width of a mixed confidence interval scales roughly (and up to constants) as $\sqrt{\nu_n \log(\nu_n)}$. If $\nu_n \approx C_n^2$, the width of the mixed confidence interval is effectively inflated by a logarithmic factor (up to constants) with respect to the confidence interval given by Theorem~\ref{corollary:fixed_bernstein}. Nonetheless, if $\nu_n$ is substantially smaller than $C_n^2$, the mixed confidence interval will be tighter than the one provided by Theorem~\ref{corollary:fixed_bernstein}. Theorem~\ref{corollary:mixed_bennett} is expected to generally work better in practice, given that the upper bound $C_n$ scales with the worst case scenario. Furthermore, the mixed confidence interval is better suited for the fully sequential setting: at time $t \ll n$, the bound is of order $\sqrt{\nu_t \log(\nu_t)} \ll \nu_n \leq C_n^2$.

% While not apparent from the probabilistic expression, the width of the confidence interval is similar to that of Theorem~\ref{corollary:fixed_bennett}, but with an extra logarithmic term of the pseudo-variance process due to the mixture argument \citep[Proposition 2]{howard2021time}. Nonetheless, it generally works better in practice, given that it is better suited for the fully sequential setting and it does not require a loose upper bound $C_n$ (see conjugate mixture discussion in \citet{howard2021time}). 

Theorem~\ref{corollary:mixed_bennett} still requires knowledge of the variance of the noises, which is unreasonable in practice. We can make it empirical by coupling it with a concentration inequality for the variance of the noises, at the expense of a small inflation of the logarithmic factor, which stems from applying the union bound to the two concentration inequalities. The result is exhibited in the following theorem, and its proof can be found in Appendix~\ref{proof:mixed_empirical_bennett}.

\begin{theorem} [Mixed empirical Bennett-type concentration bound] \label{corollary:mixed_empirical_bennett}
    Let both $(X_1, X_2, \ldots)$ and $(\epsilon_1, \epsilon_2, \ldots)$ attain Assumption \ref{assumption:adapted}, and let $\epsilon_i$ attain Assumption~\ref{assumption:bounded_variance} with constant variance, i.e.,  $|\epsilon_i| \leq B$ and $\V_{i-1} [\epsilon_i] = \sigma^2$. Denote 
    \begin{align*}
        s_t = \| \lp \rho I + V_t \rp^{-\frac{1}{2}} M_t \|, \quad \hat\nu_t = \hat\sigma_{u, t, \delta_1}^2\sum_{i \leq t}  \| G_i \|^2,
    \end{align*}
   where $\hat\sigma_{u, t, \delta_1}$ is a $1 - \delta_1$ upper confidence bound for $\sigma^2$, i.e., $\Pb \lp \hat\sigma_{u, t, \delta_1}^2 < \sigma^2 \text{ for all } t \rp \geq 1 - \delta_1$. For $\theta > 0$, the time-uniform concentration inequality \eqref{eq:mixed_bennett_eq} holds with $\delta_2$ replacing $\delta$ in the left hand side, and $\delta_1+\delta_2$ replacing $\delta$ in the right hand side.
    % \begin{align} \label{eq:gross_expression}
    %     \Pb \lp \sup_{t} \frac{\lp \frac{\theta}{B^2} \rp^\frac{\theta}{B^2}}{\Gamma \lp \frac{\theta}{B^2}\rp \overline{\gamma} \lp \frac{\theta}{B^2}, \frac{\theta}{B^2}\rp} \frac{\Gamma \lp  \frac{Bs_t + \hat\nu_t + \theta}{B^2} \rp \overline{\gamma} \lp \frac{Bs_t + \hat\nu_t + \theta}{B^2}, \frac{\hat\nu_t + \theta}{B^2} \rp}{\lp \frac{\hat\nu_t + \theta}{B^2}\rp^{\frac{Bs_t + \hat\nu_t + \theta}{B^2}} } \exp\lp{\frac{\hat\nu_t}{B^2}}\rp < \frac{2}{\delta_2} \rp \geq 1 - \delta,
    % \end{align}
    % where $\overline{\gamma}(a, x) := (\int_x^\infty u^{a-1}e^{-u}du)/\Gamma(a)$ is the regularized upper incomplete gamma function, and $\delta = \delta_1 + \delta_2$.
\end{theorem}

An upper confidence sequence for $\sigma^2$ can be obtained using the inequalities from  
\citet[Corollary 4.3]{martinez2025sharp}, which are sharper than previous results from \citet{audibert2009exploration} and \citet{maurer2009empirical}, and allow for non-constant expectation of the outcomes. 

\subsection{Comparison to existing work} \label{section:comparison_existing_work_main_body}

We present here an extended comparison of our contributions to existing work. For simplicity,  we assume for now that the conditional variances of the noises are constant, i.e. $\sigma_i = \sigma$, that the vectors $X_i$ are deterministic such that $\|X_i\| \leq L$ for all $i$, and that $\rho \geq 1 \vee L^2$. We focus on comparing our Bernstein-type concentration inequality (Theorem~\ref{corollary:fixed_bernstein}) to related works.

In particular, note that the dominating term of the confidence interval given by Theorem~\ref{corollary:fixed_bernstein} is 
\begin{align*}
C_n \sqrt{2 \log (2 / \delta)} = \sigma \sqrt{\sum_{i \leq n}\|G_i\|^2}\sqrt{2 \log (2 / \delta)}.
\end{align*}
In contrast, \citet[Theorem 1]{abbasi2011improved}  establishes an interval of 
\begin{align*}
R \sqrt{\log\frac{\det ( I + \rho^{-1} V_n  )}{\delta^2}},
\end{align*}
where $R$ is the sub-Gaussian parameter. \citet[Lemma 11]{abbasi2011improved} proves that for $G_i' := ( \rho I + V_{i-1})^{-\frac{1}{2}} X_i$, it follows that
\begin{align*}
\log \det ( I + \rho^{-1} V_n ) \leq \sum_{i \leq n} \| G_i' \|^2 \leq 2\log \det ( I + \rho^{-1} V_n ),
\end{align*}
where the second inequality holds if $\rho \geq 1 \vee L^2$. Since
\begin{align*}
\frac{1}{2}\|G_i'\|^2 \leq \frac{\rho}{\rho + L^2}\|G_i'\|^2 \leq \|G_i\|^2 \leq \|G_i'\|^2,
\end{align*}
we finally obtain that 
\begin{align*}
\frac{1}{2}\sum_{i \leq n} \| G_i \|^2 \leq \log \det ( I + \rho^{-1} V_n  ) \leq 2 \sum_{i \leq n} \| G_i \|^2.
\end{align*}
Thus, we see that our process $\sum_{i \leq n}\|G_i\|^2$ and Abbasi-Yadkori's   $\log \det ( I + \rho^{-1} V_n  )$ are within a multiplicative factor of $2$ from each other.
This also elucidates why Abbasi-Yadkori's bound is sharper when the variance and sub-Gaussian parameter are close (i.e., $R \approx \sigma$), given our slightly bigger constants.

We shall remind the reader that our inequalities are not tailored to sub-Gaussian noises as above, but rather Bernstein condition or bounded noises (where we aim to adapt to the lower variance). Given that other previous efforts have focused on the latter, we will restrict the comparison here to our bounded noise Bernstein-type inequality. First, let us note that our Theorem~\ref{corollary:fixed_bernstein} strictly improves on \citet[Theorem 4.1]{zhou2021nearly}, which has an extra logarithmic term in the sample size, on top of looser constants. In the setting of dimension-free Bernstein bounds, we also highlight three recent dimension-free contributions. \citet[Theorem 5.1]{metelli2025generalized}  provides a Bernstein-like dimension-free self-normalized concentration inequality, where the square root of $\log \det ( I + \rho^{-1} V_n  )$ appears multiplied by a factor of $\sqrt{73 \log \frac{\pi^2 (\rho + 1)^2}{3 \delta}}$, which is clearly larger than our aforementioned factor of $\sqrt{2 \log(2 /\delta)}$. \citet[Theorem 3.2]{akhavan2025bernstein} also provides a dimension-free Bernstein bound. For a fixed $\rho$, their inequality scales with $\log \det (I + \rho^{-1} V_n  )$, as opposed to its square root (thus being substantially looser for big $n$).\footnote{The dimension-free Bernstein inequality from \citet[Theorem 3.2]{akhavan2025bernstein}  is a specific instance of their more general \citet[Theorem 3.1]{akhavan2025bernstein} after invoking  \citet[Theorem 1]{faury2020improved}. While their techniques can certainly lead to other dimension-free concentration inequalities, we restrict the discussion to the inequalities explicitly established therein.} Lastly, \citet{chugg2025variational}  also developed dimension-free Bernstein bounds. While the bounds are explicit, it is not straightforward to analyze them because they depend on the inverse of a tail decay function and some ratio of eigenvalues. 

We do not argue, however, that our inequalities always dominate these recent alternative approaches. Above, we had assumed that the $X_i$ are deterministic so that the sum of $\|G_i\|^2$ was a constant, avoiding the need to assume an upper bound on it. An identical discussion would have also followed if we were seeking deterministic upper bounds for the inequalities, such as the regret analysis for the linear bandit problem exhibited in Section~\ref{section:bandits}. Nonetheless, in practical settings, one would like to adapt to the predictable process $\sum_{i \leq n}\|G_i\|^2$. While some of the alternative approaches do so inherently, we need to rely on a mixture argument (as explained in Section~\ref{section:corollaries}), and so we suffer from an extra logarithmic factor. While the magnitude of this inflation is expected to be very mild given its logarithmic nature, it implies that our inequalities do not necessarily dominate their competitors.

%% file: applications.tex
\section{Applications to online linear regression} \label{section:applications}

\subsection{Confidence ellipsoids for online linear regression}

Ellipsoids naturally appear in the context of linear regression. To be more precise, let us first revisit linear regression in the finite-dimensional case with Gaussian noise (we roughly follow the discussion presented in~\citet[Section 5.1]{whitehouse2023time} to motivate some applications of our results). That is, let $H = \R^d$, $\mathbf{Y}_t = (Y_1, \ldots, Y_t) \in \R^t$, and $\mathbf{X}_t = (X_1, \ldots, X_t)^T \in \R^{t \times d}$, such that
\begin{align*}
    \mathbf{Y}_t = \mathbf{X}_t \theta^* + \epsilon_{1:t}, \quad \epsilon_{1:t} \sim N(0, \sigma^2 I_t).
\end{align*}
The least square estimate for $\theta^* \in \R^d$ is given by $\theta_t(0)$, where $\theta_t(\rho):= (\mathbf{X}_t^T\mathbf{X}_t + \rho I)^{-1}\mathbf{X}_t^T \mathbf{Y}_t$. Assuming that $\mathbf{X}_t^T\mathbf{X}_t$ is full rank, it satisfies 
\begin{align*}
    \ldba (\mathbf{X}_t^T\mathbf{X}_t)^{\frac{1}{2}} (\theta_t(0) - \theta^*) \rdba \sim \chi_d^2.
\end{align*}
Consequently, in order to conduct inference on $\theta^*$, a confidence set can be taken to be an ellipsoid centered at $\theta_t(0)$ and thresholded at some quantile of $\chi_d^2$. Nonetheless, such an ellipsoid fails to be a nonasymptotic confidence set if certain parametric assumptions of linear regression are not attained. In contrast, we can consider a more general sequential setting without assuming homoscedastic Gaussian noises, where the samples simply attain Assumption~\ref{assumption:adapted}, and the Hilbert space $H$ is of arbitrary dimensions. In the sequential setting, confidence sets are often required to hold uniformly over time, and so the problem is conventionally termed ``online'' linear regression. Online confidence sets can be obtained from our self-normalized inequalities as exhibited in the following corollary; its proof is based on a simple triangle inequality and can be found in Appendix~\ref{proof:corollary:online_linear_regression}.

\begin{corollary} \label{corollary:online_linear_regression}
    Let $Y_t = X_t^T \theta^* + \epsilon_t$, where the random sequences $(X_t)$ and $(\epsilon_t)$  fulfill Assumption~\ref{assumption:adapted}, and $\|\theta^*\| \leq D < \infty$. If $J_t(\delta)$ is a $1-\delta$ upper confidence bound for $\ldba (\rho I + V_t)^{-1/2}M_t\rdba$ obtained from one of Theorem~\ref{corollary:fixed_bernstein}, Theorem~\ref{corollary:fixed_bennett}, Theorem~\ref{corollary:mixed_bennett}, or Theorem~\ref{corollary:mixed_empirical_bennett}, then
    \begin{align*}
        \Pb \lp \sup_t \ldba (\rho I + V_t)^{1/2}(\theta_t(\rho) - \theta^*) \rdba -J_t(\delta)  \leq \rho^{1/2}D  \rp \geq 1 - \delta. 
    \end{align*}
\end{corollary}

\subsection{Applications to (kernelized) linear bandits} \label{section:bandits}

Online linear regression has immediate applications in  (kernelized) linear bandits. In the linear bandit problem, a learner repeatedly chooses actions represented by feature vectors and observes noisy rewards that are assumed to depend linearly on an unknown parameter vector. The goal is to balance exploration (learning about the parameter) and exploitation (selecting actions with high expected reward) in order to minimize cumulative regret. The Gaussian Process Upper Confidence Bound (GP-UCB) algorithm achieves this by maintaining confidence ellipsoids around the estimated parameter and selecting the action with the highest optimistic reward estimate \citep{srinivas2009gaussian, chowdhury2017kernelized, whitehouse2023sublinear}.

Mathematically, in each round $t \in [T]$, the learner uses previous observations to select an action $X_t \in \X$, where $\X$ is a bounded subset of $H$, and then observes feedback $Y_t := \la X_t, \theta^* \ra   + \epsilon_t$. It is assumed that $\ldba \theta^*  \rdba \leq D < \infty$. The learner aims to minimize (with high probability) the regret at time $T$, which is defined as 
\begin{align*}
    R_T := \sum_{t=1}^T r_t, \quad r_t := \la x^*, \theta^* \ra - \la X_t, \theta^* \ra,
\end{align*}
where $x^* := \arg\max_{x \in \X} \la x,  \theta^* \ra$. Let
\begin{align*}
     \Pi_t(h, \eta):= \lc f \in H:  \ldba (V_t + \rho I)^{1/2}(f - h) \rdba \leq \eta\rc
\end{align*}
denote an ellipsoid in $H$ centered at $h$.  Following the optimism principle, the GP-UCB takes
\begin{align} \label{eq:action_taken}
    (X_t, \tilde \theta_t) = \arg\max_{x \in \X, f\in \Pi_{t-1}\lp \theta_{t-1}(\rho), \eta_{t-1}^{\text{subG}}\rp}  \la f,  x \ra,
\end{align}
where $\eta_t^{\text{subG}}$ is obtained from a sub-Gaussian concentration inequality \citep{abbasi2011improved}. Considering $\delta$ as a constant for simplicity, GP-UCB attains the regret bound 
\begin{align*}
    R_T = \bigo \lp B \gamma_T(\rho) \sqrt{T} + \sqrt{\rho \gamma_T(\rho) T} \rp,
\end{align*}
for $B$-bounded random noises \citep[Theorem 2]{whitehouse2023sublinear}, given that the bound is proportional to the sub-Gaussian parameter. 

We now consider variants of the GP-UCB procedure, where the threshold $\eta_t$ is obtained in view of any of our novel inequalities. For the fixed Bernstein and Bennett inequalities, the regret bound is of a similar order, but with the bound of the noises replaced by their variance. For the mixed inequalities, we obtain an extra logarithmic factor. We formalize the result in the following corollary; its proof, provided in Appendix~\ref{proof:regret_bound}, follows standard arguments. % (which are usually hidden in the analyses of regret bounds, so we do not see this as a major limitation)

\begin{corollary} \label{corollary:regret_bound}
    Let $(X_1, \ldots, X_T)$ and $(\epsilon_1, \ldots, \epsilon_T)$ attain Assumption \ref{assumption:adapted}, with $\epsilon_i$ attaining Assumption~\ref{assumption:bernstein_condition} or Assumption~\ref{assumption:bounded_variance} with constant $\sigma$. Consider variants of the GP-UCB algorithm with $\eta_t$ taken as $J_t(\delta) + \rho^{1/2}D$, where $J_t(\delta)$ is a $1-\delta$ upper confidence bound for $\ldba (\rho I + V_t)^{-1/2}M_t\rdba$ obtained from one of our self-normalized concentration inequalities. If $J_t(\delta)$ is defined following Theorem~\ref{corollary:fixed_bernstein} (if Assumption~\ref{assumption:bernstein_condition} holds) or Theorem~\ref{corollary:fixed_bennett} (if Assumption~\ref{assumption:bounded_variance} holds), then
    \begin{align*}
        R_T = \bigo \lp \sigma \gamma_T(\rho) \sqrt{T} + \sqrt{\rho \gamma_T(\rho) T} \rp.
    \end{align*}
    If Assumption~\ref{assumption:bounded_variance} holds, and $J_t(\delta)$ is defined following Theorem~\ref{corollary:mixed_bennett} or Theorem~\ref{corollary:mixed_empirical_bennett}, the above guarantee holds up to logarithmic factors.
\end{corollary}
Under Assumption~\ref{assumption:bernstein_condition} (sub-exponential noise distributions), our analysis yields regret bounds that go beyond the commonly studied bounded or sub-Gaussian noise regimes. To the best of our knowledge, such guarantees have not appeared previously in the literature. Under Assumption~\ref{assumption:bounded_variance}, these variance-dependent type bounds are usually referred to as ``second-order'' regret guarantees \citep{kirschner2018information, zhang2021variance, xu2023noise, jun2024noise, pacchiano2024second}, with the case $H=\R^d$ and constant $\rho$ being predominantly studied in the literature. Given that $\gamma_T(\rho) = \bigo(d)$ up to logarithmic factors, Corollary~\ref{corollary:regret_bound} yields a regret bound of $\bigo (d\sigma\sqrt{T} + \sqrt{dT})$ up to logarithmic factors, immediately recovering a regret bound comparable to many existing works \citep{zhou2021nearly, zhou2022computationally, kim2022improved, zhao2023variance} for constant variance. 

%Under Assumption~\ref{assumption:bernstein_condition} (sub-exponential noises), our inequalities establish regret bounds beyond sub-Gaussian or bounded noises, thus widely generalizing the result. 

%For bounded noises, sub-Gaussian-type inequalities make $J_T(\delta)$, which is one of the two terms in $U_t(\delta)$, scale roughly as $B\sqrt{\gamma_T(\rho)}$ (given that $B$ is proportional to the sub-Gaussian parameter). The regret guarantee will thus scale linearly with $B$. In contrast, our Bernstein-type inequality make $U_T(\delta)$ scale roughly as $\sigma\sqrt{\gamma_T(\rho)}$ (dominant term in $C_T$). We obtain the same behavior as the sub-Gaussian case but with the bound replaced by the variance, yielding `second-order' regret guarantees. \dmtcomment{Compare to existing results in the literature. Be more precise in general. We suggest the following variant of the UCB algorithm. Make a comment about $\eta$. Make $U_t(\delta)$ more clear.}

\subsection{Experiments} \label{section:experiments}

In order to elucidate the empirical differences between our variance-dependent inequalities and the predominant sub-Gaussian ones, we run an ablation study on the GP-UCB algorithm for the linear bandit problem using the sub-Gaussian inequality from \citet{abbasi2011improved}. We also evaluate our mixed Bennett inequality and our empirical mixed Bennett inequality. The code can be found at \href{https://github.com/DMartinezT/self_normalized}{https://github.com/DMartinezT/self\_normalized}. %The covariates are RBF kernel $k$-embeddings of one-dimensional points, and the true regression function is given by a weighted sum of the embeddings of different points. The outcomes correspond to the evaluation of the regression function added to random noise, which follows a rescaled  $(10, 10)$-beta distribution.  We defer further experimental details to Appendix~\ref{section:experiment_details}. %in order to yield the confidence ellipsoids 

More specifically, we consider a bandit experiment, where at each round $t$ an action $X_t$ is taken following the UCB procedure introduced in \eqref{eq:action_taken}. The covariates are RBF kernel embeddings of one-dimensional points, 
where the kernel length scale set to $0.01$. The bound of the kernel is naturally $1$. We take $\rho = 0.05$, $\delta = 2\delta_1 = 2\delta_2 = 0.1$, and the mixing hyperparameter $\theta = 1$. The true regression function is given by a weighted sum of $50$ embeddings of different points, which are randomly drawn concentrated around two modes (such that it has two local maxima for the sake of visualization). 
For the empirical mixed Bennett-type inequality, $\hat\sigma_{u, t, \delta_1}^2$ is obtained using the upper confidence sequence from \citet{martinez2025sharp} with $\hat\mu_i$ and $\bar\mu_i$ taken  as the evaluation of $\theta_t(\rho)$ on $X_i$.
No effort has been put into optimizing the hyperparameter choices. % The outcomes correspond to the evaluation of the regression function added to random noise, which follows a $(10, 10)$-beta distribution rescaled to the interval $(-1, 1)$. 

% Figure~\ref{fig:experiments} illustrates the upper confidence bounds derived using three different concentration inequalities after $500$ rounds. 

\begin{figure}[ht] 
    \centering
  \includegraphics[width=\textwidth]{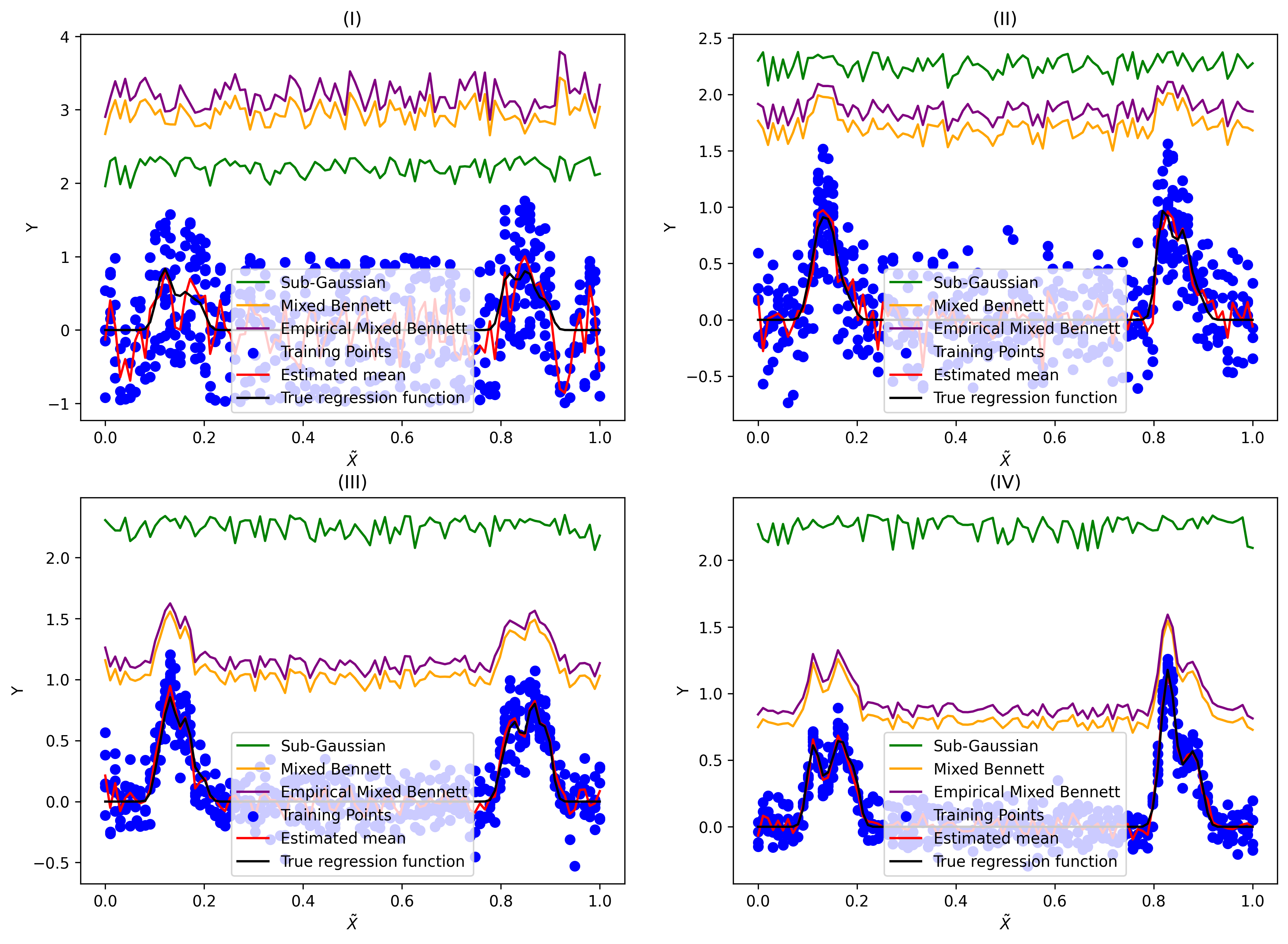}
  \caption{Illustration of the optimistic upper confidence bounds for the regression function after $500$ rounds using sub-Gaussian, mixed Bennett, and empirical mixed Bennett inequalities for noises following (I) a rescaled uniform distribution, (II) a rescaled ($5, 5$)-beta distribution, (III) a rescaled ($20, 20$)-beta distribution, and (IV) a rescaled ($50, 50$)-beta distribution. Training points are drawn following a UCB procedure, with the covariates $X_t = k(\cdot, \tilde X_t)$ illustrated in the original space (pre-embedded in the RKHS). }
  \label{fig:experiments_four}
\end{figure}

We consider four experimental settings, where the outcomes correspond to the evaluation of the regression function added to different random noise, which follows either a uniform distribution or beta distributions with different parameters. All of them are rescaled to lie in the interval $(-1, 1)$. Figure~\ref{fig:experiments_four} illustrates the true regression function, the training points, and the estimated mean of the bandit algorithm after $500$ rounds. We generally see that the smaller the variance (with respect to the scale of the regression function and the bound of the noise), the better our inequalities perform in comparison to the sub-Gaussian inequalities from \citet{abbasi2011improved}. In particular, we observe that for comparatively small variance, i.e. plots (II), (III), and (IV), our inequalities lead to sharper bounds of the regression function, elucidating the empirical gains of using variance-aware inequalities. In contrast, our inequalities are empirically outperformed by the sub-Gaussian approach for uniform-distributed noises (comparatively larger
variance), where the difference between the variance and bound is not large enough to justify the extra logarithmic term of our mixed inequalities.

% Figure~\ref{fig:experiments}  On top of that, Figure~\ref{fig:experiments} displays the optimistic upper bound for the original (sub-Gaussian) UCB, as well as  our mixed and empirical mixed Bennett inequalities optimistic upper bounds. We observe that our inequalities lead to sharper optimistic upper bounds of the regression function, elucidating the empirical gains of using variance-aware inequalities.  

% \begin{figure}[ht] 
%     \centering
%   \includegraphics[width=.95\textwidth]{plots/visualize_selection_beta1010.png}
%   \caption{Illustration of the optimistic upper confidence bounds for the regression function using sub-Gaussian, mixed Bennett, and empirical mixed Bennett inequalities for noises following a rescaled ($10, 10$)-beta distribution. Training points are drawn following a UCB procedure, with the covariates $X_t = k(\cdot, \tilde X_t)$ illustrated in the original space (pre-embedded in the RKHS). }
%   \label{fig:experiments}
% \end{figure}

%% file: conclusion.tex
\section{Conclusion} \label{section:conclusion}

We have proposed novel concentration inequalities for vector-valued self-normalized processes. These include anytime-valid Bernstein and Bennett type inequalities tailored to a fixed sample size (Proposition~\ref{corollary:fixed_bernstein} and Proposition~\ref{corollary:fixed_bennett}, respectively), a mixed Bennett-type inequality that is well suited for randomly stopped sample sizes, and an empirical mixed Bennett-type inequality that does not require knowledge of the variance in advance. These inequalities build on the theoretical tools from \citet{pinelis1994optimum}, thus being fundamentally different to previous vector-valued self-normalized inequalities. We have further explored the immediate consequences of our inequalities in the (kernelized) bandit setting, both theoretically in the form of second order regret bounds and empirically. 

There are several directions for future work. First, the proposed inequalities can have interesting applications beyond linear bandits. Natural extensions may include reinforcement learning \citep{yang2020function, vakili2024kernelized}, safe Bayesian optimization \citep{chowdhury2017kernelized, amani2020regret}, and autoregressive models  \citep{darolles2006structural, pacurar2008autoregressive, shao2015self, agarwal2018model}. Second, our empirical Bennett inequality assumes constant variance for the noise, so we can exploit the upper confidence sequences from \citet[Corollary 4.3]{martinez2025sharp}. However, there is evident interest in bandit algorithms that adapt to heteroscedastic noises \citep{kirschner2018information, kim2022improved, zhao2023variance}. Extending \citet[Corollary 4.3]{martinez2025sharp} to a covariate-dependent inequality would immediately yield heteroscedastic-noise guarantees when in conjunction with Proposition~\ref{corollary:mixed_bennett}. Lastly, our inequalities could potentially be sharpened if the loose bound $V_{t-1} \preceq V_t$ was to be avoided in the supermartingale construction (see proof of Theorem~\ref{theorem:main_theorem}). This limitation may inflate the growth of the inequalities by a logarithmic term (see Appendix~\ref{section:limitations} for a detailed presentation), which implies that sub-Gaussian inequalities are sharper than our mixed inequalities if the variance of the noise is not substantially smaller than the bound of the noise. Finding a refined construction that circumvents this limitation constitutes an important and difficult open direction for future work.

%% file: proofs.tex
\section{Proofs} \label{section:proofs}

\subsection{Proof of Proposition \ref{corollary:main_cor}} \label{proof:main_cor}

Extend $S_t$ to $t = 0$ with $\lambda_0 = 0$ and $X_0 = 0$. It follows that $S_t$ is a nonnegative supermartingale with $S_0 = 1$. Consequently, Ville's inequality (Fact \ref{fact:villesineq}) yields 
\begin{align*}
    \Pb \lp \sup_{t } S_t \geq \frac{1}{\delta} \rp \leq \E[S_0] \delta = \delta.
\end{align*}

Given that $e^u \leq 2 \cosh(u)$ for all $u \in \R$,  it follows from Theorem \ref{theorem:main_theorem} that %and $e^{-u} \leq (1+u)^{-1}$ for all $u \geq 0$,
\begin{align*}
    \frac{1}{2}\exp \lp \ldba  \lambda (\rho I + V_t)^{-1/2}M_t  \rdba \rp  \exp\lp -\sum_{i = 1}^te_i (\lambda) \rp
\end{align*}
is dominated by $S_t$. Thus, with probability $1-\delta$, and simultaneously for all $t \geq 1$,
\begin{align*}
   \frac{1}{2}\exp \lp \ldba \lambda (\rho I + V_t)^{-1/2}M_t  \rdba \rp  \exp\lp -\sum_{i = 1}^te_i(\lambda)  \rp \leq \frac{1}{\delta}.
\end{align*}

Taking logarithms and dividing both sides by $\lambda$, it follows that
\begin{align*}
        \ldba (\rho I + V_t)^{-1/2}M_t \rdba  \leq \frac{ \sum_{i=1}^t e_i (\lambda) + \log \lp \frac{2}{\delta}\rp}{ \lambda}.
    \end{align*}

\subsection{Proof of Corollary~\ref{corollary:lil}} \label{proof:lil}
Throughout, we refer to a stochastic process as $l_0$-sub-$\psi$ following \citet[Definition 1]{howard2021time}. Denote
\begin{align*}
   \psi_{G, B}(\lambda) := \frac{\lambda^2}{2(1 - \lambda B)}\sigma_t^2 , \quad 0 \leq \lambda < \frac{1}{B}; \quad \psi_{P, B}(\lambda) :=  \frac{e^{\lambda B} - \lambda B - 1}{B^2}, \quad \lambda \geq 0.
\end{align*}
Based on $e^x \leq 2 \cosh x$, Theorem~\ref{theorem:main_theorem} implies that 
\begin{align*}
    \tilde S_t = \exp \lp \lambda \ldba (\rho I + V_t)^{-1/2}M_t  \rdba  - \sum_{i \leq t} e_i(\lambda)    \rp
\end{align*}
is dominated by the nonnegative supermartingale $2S_t$. If Assumption~\ref{assumption:bernstein_condition} holds, then
\begin{align*}
    \sum_{i \leq t} e_i(\lambda) \leq \psi_{G, B}(\lambda) \sum_{i \leq t} \|G_i\|^2 \sigma_i^2,
\end{align*}
and $\ldba (\rho I + V_t)^{-1/2}M_t  \rdba$ is 2-sub-$\psi_{G, B}$ with variance process $\sum_{i \leq t} \|G_i\|^2 \sigma_i^2$. If  Assumption~\ref{assumption:bounded_variance} holds,
\begin{align*}
    \sum_{i \leq t} e_i(\lambda) \leq \psi_{P, B}(\lambda) \sum_{i \leq t} \|G_i\|^2 \sigma_i^2,
\end{align*}
and $\ldba (\rho I + V_t)^{-1/2}M_t  \rdba$ is 2-sub-$\psi_{P, B}$ with variance process $\sum_{i \leq t} \|G_i\|^2 \sigma_i^2$. In either case, \citet[Corollary 1]{howard2021time} can be applied to conclude the result, in view of $\psi_{G, B}(\lambda) \approx \lambda^2 / 2 \approx \psi_{P, B}(\lambda)$ as $\lambda \downarrow 0$.

\subsection{Proof of Theorem~\ref{corollary:fixed_bernstein}} \label{proof:fixed_bernstein}

It follows from Proposition~\ref{corollary:main_cor} that
\begin{align*}
    \ldba (\rho I + V_t)^{-1/2}M_t \rdba  \leq \frac{ \sum_{i=1}^t e_i(\lambda) + \log \lp \frac{2}{\delta}\rp}{\lambda}.
\end{align*}
simultaneously for all $t \geq 1$ with probability $1 - \delta$. As observed in Section~\ref{section:main_theorem}, 
\begin{align*}
    e_i(\lambda) \leq \frac{\lambda^2}{2(1 - \lambda B)}\|G_i\|^2\sigma_i^2 
\end{align*}
for $\lambda \in (0, \frac{1}{B})$, and so 
\begin{align*}
    \sup_{t \leq n}\ldba (\rho I + V_t)^{-1/2}M_t \rdba  &\leq \sup_{t \leq n} \frac{ \frac{\lambda^2}{2(1 - \lambda B)}\sum_{i=1}^t \sigma_i^2 \|G_i\|^2 + \log \lp \frac{2}{\delta} \rp }{\lambda}
    \\ &\leq \frac{ \frac{\lambda^2}{2(1 - \lambda B)}C_n^2 + \log \lp\frac{2}{\delta}\rp}{\lambda}.
\end{align*}
Now we optimize over $\lambda \in (0,1/B)$. Denote $L := \log \lp\frac{2}{\delta} \rp$. Consider the function
\[
f(\lambda) := \frac{ \tfrac{\lambda^2}{2(1 - \lambda B)} C_n^2 + L}{\lambda}
= \frac{C_n^2\lambda}{2(1-\lambda B)} + \frac{L}{\lambda}, 
\quad \lambda \in (0,1/B).
\]
Writing \(f\) in terms of \(x:=\lambda^{-1}\) gives
\[
f\big(\tfrac{1}{x}\big)
= \frac{C_n^2}{2(x-B)} + L x,\quad x>B.
\]
Differentiating with respect to $x$ yields
\[
\frac{d}{dx}\!\Big( \frac{C_n^2}{2(x-B)} + L x\Big)
= -\frac{C_n^2}{2(x-B)^2} + L,
\]
and equaling it to $0$ leads to the minimizer 
\[
x^\ast = B + \sqrt{\frac{C_n^2}{2L}}.
\]
Thus
\[
\begin{aligned}
\min_{0<\lambda<1/B} f(\lambda)
&= f\big(\tfrac{1}{x^\ast}\big)
= L x^\ast + \frac{C_n^2}{2(x^\ast-B)}
\\&= L\Big(B + \sqrt{\frac{C_n^2}{2L}}\Big) + \frac{C_n^2}{2\sqrt{C_n^2/(2L)}}
\\&= B L + \sqrt{2 C_n^2 L}.
\end{aligned}
\]
We thus conclude that
\begin{align*}
    \sup_{t \leq n}\ldba (\rho I + V_t)^{-1/2}M_t \rdba  &\leq B \log \lp \frac{2}{\delta} \rp + \sqrt{2 C_n^2 \log \lp \frac{2}{\delta} \rp }.
\end{align*}

\subsection{Proof of Theorem~\ref{corollary:fixed_bennett}} \label{proof:fixed_bennett}

It follows from Proposition~\ref{corollary:main_cor} that
\begin{align} \label{eq:pre-bennett}
    \ldba (\rho I + V_t)^{-1/2}M_t \rdba  \leq \frac{ \sum_{i=1}^t e_i(\lambda) + \log \lp \frac{2}{\delta} \rp}{\lambda}.
\end{align}
simultaneously for all $t \geq 1$ with probability $1 - \delta$. As observed in Section~\ref{section:main_theorem}, 
\begin{align*}
    e_i(\lambda) \leq \frac{e^{\lambda B} - \lambda B - 1}{B^2}\sigma_i^2 \|G_i\|^2,
\end{align*}
and so 
\begin{align*}
    \sup_{t \leq n}\ldba (\rho I + V_t)^{-1/2}M_t \rdba  &\leq \sup_{t \leq n} \frac{ \frac{e^{\lambda B} - \lambda B - 1}{B^2}\sum_{i=1}^t \sigma_i^2 \|G_i\|^2 + \log \lp \frac{2}{\delta}\rp}{\lambda}
    \\ &\leq \frac{ \frac{e^{\lambda B} - \lambda B - 1}{B^2}C_n^2 + \log \lp \frac{2}{\delta} \rp }{\lambda}.
\end{align*}
Denote
\[
L := \log \lp\frac{2}{\delta}\rp,
\quad s := B\lambda,
\quad \phi(s) := e^s - s - 1,
\quad A := \frac{C_n^2}{B^2}.
\]
It follows that %\eqref{eq:pre-bennett} becomes
\[
\sup_{t \leq n} \ldba (\rho I + V_t)^{-1/2} M_t \rdba
\leq B g(s), 
\quad
g(s) := \frac{A\phi(s) + L}{s}.
\]
The best bound is obtained by minimizing $g(s)$ over $s>0$. Differentiating gives
\[
g'(s) 
= \frac{A \phi'(s)s - (A \phi(s)+L)}{s^2}.
\]
Setting $g'(s)=0$ yields $A s \phi'(s) = A\phi(s) + L$. Recalling $\phi(s)=e^s-s-1$ and $\phi'(s)=e^s-1$, this becomes
\[
s(e^s-1) = e^s - s - 1 + \frac{L}{A}.
\]
Let now $u := e^s - 1$, so $s = \log(1+u)$.  Substituting and simplifying leads to
\[
h(u)  = \frac{B^2}{C_n^2} L,
\]
where $h(u) = (1 + u) \log (1 + u) - u$. Let $u^*$ denote such an optimal value for $u$, and analogously for $s^*$. Using that $A s^* \phi'(s^*) = A\phi(s^*) + L$, we compute
\[
g(s^*)=\frac{A\phi(s^*)+L}{s^*}=\frac{A s^* \phi'(s^*)}{s^*}=A\phi'(s^*)=A(e^{s^*}-1)=A u^*.
\]
Therefore the optimized upper bound %right-hand side in \eqref{eq:pre-bennett} 
evaluates to 
\[
Bg(s^*) = B A u^* = \frac{C_n^2}{B}u^*.
\]
Recalling that $u^*$ is the unique solution of \(h(u)=\dfrac{B^2}{C_n^2}\log\frac{2}{\delta}\), we conclude that 
\begin{align*}
    \sup_{t \leq n}\ldba (\rho I+V_t)^{-1/2}M_t\rdba
\leq \frac{C_n^2}{B} h^{-1}\lp\frac{B^2}{C_n^2}\log\lp \frac{2}{\delta}\rp\rp
\end{align*}
 holds with probability at least \(1-\delta\).

\subsection{Proof of Theorem~\ref{corollary:mixed_bennett}} \label{proof:mixed_bennett}

The proof of this corollary primarily relies on the combination of Theorem~\ref{theorem:main_theorem} and the Gamma-Poisson mixture argument from \citet[Proposition 10]{howard2021time}. 

Given that $e^u \leq 2 \cosh(u)$ for all $u \in \R$,  it follows that
\begin{align*}
    \mathfrak{e}_t := \exp \lp \ldba  \lambda(\rho I + V_t)^{-1/2}M_t  \rdba   -\sum_{i = 1}^te_i (\lambda) \rp \leq 2 S_t,
\end{align*}
where $S_t$ is the nonnegative supermatingale established in Theorem~\ref{theorem:main_theorem} with $S_0 = 1$. 
Denoting $\psi_{P, B}(\lambda) := B^{-2}(e^{\lambda B} - \lambda B - 1)$, and in view of 
\begin{align*}
    e_i (\lambda) \leq \psi_{P, B}(\lambda)\sigma_i^2 \|G_i\|^2,
\end{align*}
the process $\mathfrak{e}_t$ falls under \citet[Definition 1]{howard2021time} as $2$-sub-$\psi_{P, B}$ with variance process $\sum_{i\leq t}\sigma_i^2 \|G_i\|^2$. Thus, \citet[Proposition 10]{howard2021time} can be invoked to yield the corollary.

\subsection{Proof of Theorem~\ref{corollary:mixed_empirical_bennett}} \label{proof:mixed_empirical_bennett}

The corollary is obtained in view of Theorem~\ref{corollary:mixed_bennett} and a union bound. Let $A$ be the event that 
\begin{align*}
        \hat\sigma_{u, t, \delta_1}^2 < \sigma^2
\end{align*}
for some $t \geq 0$, and C the event that
\begin{align} \label{eq:gross_expression_proof}
    \frac{\lp \frac{\theta}{B^2} \rp^\frac{\theta}{B^2}}{\Gamma \lp \frac{\theta}{B^2}\rp \overline{\gamma} \lp \frac{\theta}{B^2}, \frac{\theta}{B^2}\rp} \frac{\Gamma \lp  \frac{Bs_t + \nu_t + \theta}{B^2} \rp \overline{\gamma} \lp \frac{Bs_t + \nu_t + \theta}{B^2}, \frac{\nu_t + \theta}{B^2} \rp}{\lp \frac{\nu_t + \theta}{B^2}\rp^{\frac{Bs_t + \nu_t + \theta}{B^2}} } \exp\lp{\frac{\nu_t}{B^2}}\rp > \frac{2}{\delta_2}
\end{align}
for some $t \geq 0$, where $\nu_t = \sigma^2\sum_{i \leq t}  \| G_i \|^2$.

We observe that $\Pb (A) \leq \delta_1$ by assumption, and $\Pb (C) \leq \delta_2$ by Theorem~\ref{corollary:mixed_bennett}. Thus $P(A \cup C) \leq \delta_1 + \delta_2 \leq \delta$ in view of the union bound. Thus, $P(\bar A \cap \bar C) = 1 - P( A \cup C) \geq 1 - \delta$.

Denote the LHS of \eqref{eq:gross_expression_proof} by $\mathfrak{e}_t$, and its empirical counterpart
\[ \frac{\lp \frac{\theta}{B^2} \rp^\frac{\theta}{B^2}}{\Gamma \lp \frac{\theta}{B^2}\rp \overline{\gamma} \lp \frac{\theta}{B^2}, \frac{\theta}{B^2}\rp} \frac{\Gamma \lp  \frac{Bs_t + \hat\nu_t + \theta}{B^2} \rp \overline{\gamma} \lp \frac{Bs_t + \hat\nu_t + \theta}{B^2}, \frac{\hat\nu_t + \theta}{B^2} \rp}{\lp \frac{\hat\nu_t + \theta}{B^2}\rp^{\frac{Bs_t + \hat\nu_t + \theta}{B^2}} } \exp\lp{\frac{\hat\nu_t}{B^2}}\rp\]
by $ \widehat {\mathfrak{e}_t}$. Since both expressions are obtained as mixtures of functions that are decreasing on $\sigma$, if $\bar A$ holds, then $\widehat{\mathfrak{e}_t} \leq \mathfrak{e}_t$. Furthermore, if $\bar C$ holds, then $\sup_t \mathfrak{e}_t \leq \frac{2}{\delta_2}$. Together, these imply that $\sup_t\widehat{\mathfrak{e}_t} \leq \frac{2}{\delta_2}$ with probability at least $1 - \delta$.

\subsection{Proof of Corollary~\ref{corollary:online_linear_regression}} \label{proof:corollary:online_linear_regression}
It suffices to observe that
\begin{align*}
        \ldba (V_t + \rho I)^{1/2}(\theta_t(\rho) - \theta^*) \rdba  &= \ldba (V_t + \rho I)^{1/2} \lp(V_t + \rho I)^{-1} \mathbf{X}_t^T \epsilon_{1:t} - \rho(\rho I + V_t)^{-1} \theta^* \rp \rdba
        \\&\leq \ldba  (V_t + \rho I)^{-\frac{1}{2}} \mathbf{X}_t^T \epsilon_{1:t} \rdba + \ldba \rho(\rho I + V_t)^{-\frac{1}{2}} \theta^*  \rdba\\
        &=\ldba  (V_t + \rho I)^{-\frac{1}{2}} M_t \rdba + \ldba \rho(\rho I + V_t)^{-\frac{1}{2}} \theta^*  \rdba
        \\&\leq J_t(\delta) + \rho^{1/2}D. 
\end{align*}

\subsection{Proof of Corollary~\ref{corollary:regret_bound}} \label{proof:regret_bound}
Define $\eta_0(\delta) = \rho^{1/2}D, \theta_0(\rho) = 0$ and consider $V_0 = 0$, so that $\Pi_0\lp \theta_{0}(\rho), \eta_{0}\rp$ is the ball centered at $0$ of radius $D$ containing $\theta^*$. For $t\ge 1$, let $\eta_t(\delta) = J_t(\delta) + \rho^{1/2}D$, where $J_t(\delta)$ is obtained from one of our concentration inequalities. Taking
\begin{align*}
    (X_t, \tilde \theta_t) = \arg\max_{x \in \X, f\in \Pi_{t-1}\lp \theta_{t-1}(\rho), \eta_{t-1}\rp}  \la f,  x \ra,
\end{align*}
the regret can be upper bounded with probability $1-\delta$ as
\begin{align*}
    r_t &= \la \theta^*,x^* \ra - \la \theta^*, X_t\ra \stackrel{(i)}{\leq} \la \tilde \theta_t, X_t \ra - \la \theta^*, X_t\ra 
    \\&= \la \tilde \theta_t -  \theta_{t-1}(\rho), X_t \ra - \la  \theta_{t-1}(\rho) - \theta^*, X_t \ra
    \\&\stackrel{(ii)}{\leq} \ldba \lp V_{t-1} + \rho I \rp^{-1/2} X_t \rdba \lp \ldba \lp V_{t-1} + \rho I \rp^{1/2} \lp \tilde\theta_t -  \theta_{t-1}(\rho)\rp \rdba + \ldba \lp V_{t-1} + \rho I \rp^{1/2} (\theta_{t-1}(\rho) - \theta^*)\rdba \rp
    \\&\stackrel{(iii)}{\leq} 2\eta_{t-1}(\delta)\ldba \lp V_{t-1} + \rho I \rp^{-1/2} X_t \rdba,
\end{align*}
where (i) follows from the definition of $(X_t,\tilde \theta_t)$ together with the fact that we are considering a high probability event where $\theta^* \in \Pi\lp \theta_t(\rho), \eta_t\rp$ for every $t\ge 0$, (ii) follows from Cauchy-Schwarz inequality, and (iii) follows from the definition of the ellipsoid $\Pi\lp \theta_{t-1}(\rho), \eta_{t-1}\rp$ along with the fact that both $\tilde \theta_t$ and $\theta^*$ belong to it. %if $J_t(\delta)$ holds with anytime validity. 
Hence, 
\begin{align*}
    R_T &= \sum_{t=1}^T r_t \stackrel{(i)}{\leq} \sqrt{T \sum_{t=1}^T r_t^2} \\
    &\stackrel{(ii)}{\leq} \sqrt{T \sum_{t=1}^T 4 \eta_{t-1}(\delta)^2 \ldba \lp V_{t-1} + \rho I \rp^{-1/2} X_t \rdba^2}
    \\&\stackrel{(iii)}{\leq} 4\eta_{T}(\delta) \sqrt{T \gamma_T (\rho)} 
\end{align*}
with probability $1 - \delta$, where (i) follows from the Cauchy-Schwarz inequality, (ii) follows from the elliptical potential lemma and (iii) is obtained given that $t \mapsto \eta_t(\delta)$ is non-decreasing. 

If $\eta_T(\rho)$ is obtained from Theorem~\ref{corollary:fixed_bernstein}, then
\begin{align*}
    J_T(\rho)  \leq B \log \lp \frac{2}{\delta} \rp + \sigma \sqrt{\gamma_T(\rho)}\sqrt{2  \log \lp \frac{2}{\delta}\rp}= \bigo \lp \sigma\sqrt{\gamma_T(\rho)}\rp.
\end{align*}

If $\eta_T(\rho)$ is obtained from Theorem~\ref{corollary:fixed_bennett}, it is not closed form. However, it is well known that
\begin{align*}
    h^{-1}(y) \leq \sqrt{2y} + \frac{y}{3},
\end{align*}
and so it follows that $J_T(\rho)$ is upper bounded by
\begin{align*}
    C_T\sqrt{2 \log(2/\delta)} + \frac{B}{3} \log(2/\delta) \leq \sigma\sqrt{8 \gamma_T(\rho) \log(2/\delta)} + \frac{B}{3} \log(2/\delta) = \bigo \lp \sigma\sqrt{\gamma_T(\rho)}\rp,
\end{align*}
where in the first inequality we used that $C_T := \sigma \sqrt{4\gamma_T(\rho)}$ is an upper bound on $\sigma \sqrt{\sum_{t \leq T}\|G_i\|^2}$ by definition of $\gamma_T(\rho)$. 
Consequently,
\begin{align*}
    \eta_T(\rho)
    = J_T(\rho) + \rho^{1/2} D = \bigo \lp \sigma\sqrt{\gamma_T(\rho)} + \sqrt\rho \rp, 
\end{align*}
which implies
\begin{align*}
    R_T = \bigo \lp \lp \sigma\sqrt{\gamma_T(\rho)} + \sqrt\rho \rp  \sqrt{T \gamma_T (\rho)} \rp.
\end{align*}

If $J_T(\rho)$ is obtained from Theorem~\ref{corollary:mixed_bennett}, \citet[Proposition 2]{howard2021time} implies that $J_T(\rho)$ is also $\bigo (\sigma\sqrt{\gamma_T(\rho)})$ up to logarithmic factors, from which the same regret bound (up to logarithmic factors) follows. Lastly, if $J_T(\rho)$ is obtained from Theorem~\ref{corollary:mixed_empirical_bennett}, and $\hat\sigma_{u, t, \delta_1}$ is $\sigma(1 + o(1))$ with  high probability, then the same regret bound holds. The $\sigma(1 + o(1))$ condition holds for the inequalities from \cite{martinez2025sharp}, with \citet[Section 4.4]{martinez2025sharp} establishing that $\hat\sigma_{u, T, \delta_1} \lessapprox \sigma + c/\sqrt{T}$ for some constant $c$.

%% file: related_work.tex
\section{Extended related work} \label{section:related_work}

\paragraph{Self-normalized scalar processes.} A prominent line of research concerns self-normalized concentration inequalities developed by \citep{de2004self, de2007pseudo, pena2009self}, which establish time-uniform guarantees on the behavior of self-normalized scalar processes. These results are obtained via the method of mixtures, a probabilistic technique originally introduced by Robbins \citep{darling1967iterated, darling1968some}, which constructs bounds by averaging over a parameterized family of exponential supermartingales. Building on this framework, \citet{bercu2008exponential} further explored the self-normalized regime, deriving concentration inequalities accommodating asymmetric and heavy-tailed increment distributions. Later on, they extended this analysis by incorporating both predictable and empirical quadratic variations \citep{bercu2019new}.

\paragraph{Sub-Gaussian self-normalized vector processes.} Going from one dimension to several or infinite dimensions is far from straightforward. For this reason, most of the advances for self-normalized processes are in the sub-Gaussian case, where mixture methods provide clean concentration inequalities \citep{abbasi2011improved, chowdhury2017kernelized, flynn2023improved, flynn2024tighter}. We highlight the seminal work by \citet{abbasi2011improved}, which was extended to Hilbert spaces by \citet{abbasi2013online} (see also \citet{whitehouse2023sublinear}). \citet{chowdhury2017kernelized} also provided a related (though inferior) concentration inequality using a `double mixture' technique.  

\paragraph{Self-normalized vector processes beyond sub-Gaussianity.} 

 \citet{victor2009theory}  worked out some multivariate inequalities in more general regimes. However, these are not closed form and their theoretical properties hard to study. More recently, \citet{whitehouse2023time} presented tractable self-normalized inequalities for general light-tailed noises; however, their argument relies on a covering argument that is dimension dependent and not generalizable to infinite dimensions. Similarly, \citet{ziemann2024vector} provided a self-normalized vector Bernstein inequality that is also dimension dependent and restricted to finite dimensional spaces, via PAC-Bayes arguments. Concurrent contributions explored dimension-independent self-normalized inequalities.  \citet{chugg2025variational} developed inequalities via PAC-Bayes;  while the bounds are explicit, it is not straightforward to analyze the rate because they depend on the inverse of a tail decay function and some ratio of eigenvalues. \citet[Theorem 5.1]{metelli2025generalized} developed Bernstein-like concentration inequalities for bounded noises relying on stitching arguments. \citet{akhavan2025bernstein} leveraged method of mixtures and truncation arguments to also develop a Bernstein-like concentration inequality. Our approach is fundamentally different to these recent works, building on the tools originally introduced in~\citet{pinelis1994optimum}. In a different line of research, \citet[Theorem 4.1]{zhou2021nearly} also presented a Bernstein-type self-normalized concentration inequality for vector-valued martingales, whose proof exploits solely univariate concentration inequalities (similarly to \citet[Theorem 5]{dani2008stochastic}). In contrast to our Bernstein-type inequality, this results in substantially looser constants and an extra logarithmic term in the sample size. Furthermore, our main theorem is significantly more general, leading to Bennett-type inequalities and generalizing to unbounded random variables. Follow-up work  sharpened the inequalities in specific scenarios \citep{zhou2022computationally, he2023nearly, zhao2023variance}, e.g. weighted linear regression, which are out of the scope of this contribution.  %, requiring to work with extra terms and to make use of some union bound.

\paragraph{Light-tailed vector-valued concentration inequalities.} In the context of sums of random vectors, \citet{pinelis1992approach, pinelis1994optimum} introduced a martingale based approach tailored to light-tailed random vectors, which led to generalizations of well-known concentration inequalities (such as  Hoeffding and Bernstein inequalities) that hold uniformly over time in smooth Banach spaces. In his framework, any dependence on dimensionality is effectively substituted by a geometric property of the underlying Banach space, i.e. its smoothness parameter (which equals one in Hilbert spaces). Recent vector-valued concentration bounds, such as sharp vector-valued empirical Bernstein inequalities \citep{martinez2024empirical, martinez2025sharp} or heavy-tailed vector-valued concentration inequalities \citep{whitehouse2024mean} build on the theoretical tools introduced by Pinelis. However, these were not self-normalized, and our contribution pushes on this trajectory by generalizing the Pinelis framework to self-normalized processes.

\paragraph{Time-uniform Chernoff bounds.} The use of nonnegative supermartingale techniques to derive concentration inequalities has gained significant traction to provide probabilistic guarantees for streams of data that are continuously
monitored and adaptively stopped, with Ville's inequality~\citep{ville1939etude} being the theoretical pillar of this line of research. The results presented in this work fall within the broader umbrella of time-uniform concentration, aligning with the anytime-valid Chernoff-style bounds exhibited in \citet{howard2020time, howard2021time}.

% \paragraph{Second order bounds for online regression.} While the focus of this contribution is to develop novel concentration inequalities,  Section~\ref{section:applications} exhibits the implications of our Bernstein-type self-normalized concentration inequality in the (kernelized) linear bandit problem. There have been extensive efforts in developing variance-dependent bounds for linear bandits \citep{kirschner2018information, zhang2021variance, zhou2021nearly, kim2022improved, zhou2022computationally, zhao2023variance, xu2023noise, jun2024noise, pacchiano2024second}, and we recover some of the state-of-the-art second order (i.e., variance dependent) regret bounds using the proposed concentration inequalities. In contrast to some of the previous efforts, our inequalities are clean and the analysis is straightforward.% Furthermore, these previous works rely on arguably complex and ad-hoc procedures, e.g., the SAVE procedure \citep{zhao2023variance} has to maintain disjoint buckets of the observed samples given that it is inherently based on SupLinRel \citep{auer2002using} or SupLinUCB \citep{chu2011contextual}. 

%% file: extended.tex
\section{Further experiments} \label{section:further_experiments}

In Section~\ref{section:experiments}, we presented an empirical comparison between our inequalities and those of \citet{abbasi2011improved}. We extend these experiments to additionally include the inequalities proposed in \citet[Theorem~3.2]{akhavan2025bernstein} and \citet[Theorem~5.1]{metelli2025generalized}. All hyperparameter choices are kept identical to those used in Section~\ref{section:experiments}. Figure~\ref{fig:experiments_recent_work} reports the resulting confidence bounds corresponding to the different concentration inequalities. We observe that our inequalities, as well as those of \citet{abbasi2011improved}, yield substantially tighter bounds in practice than these more recent alternatives. This behavior is consistent with the theoretical considerations discussed in Section~\ref{section:comparison_existing_work_main_body}.

\begin{figure}[ht] 
    \centering
  \includegraphics[width=\textwidth]{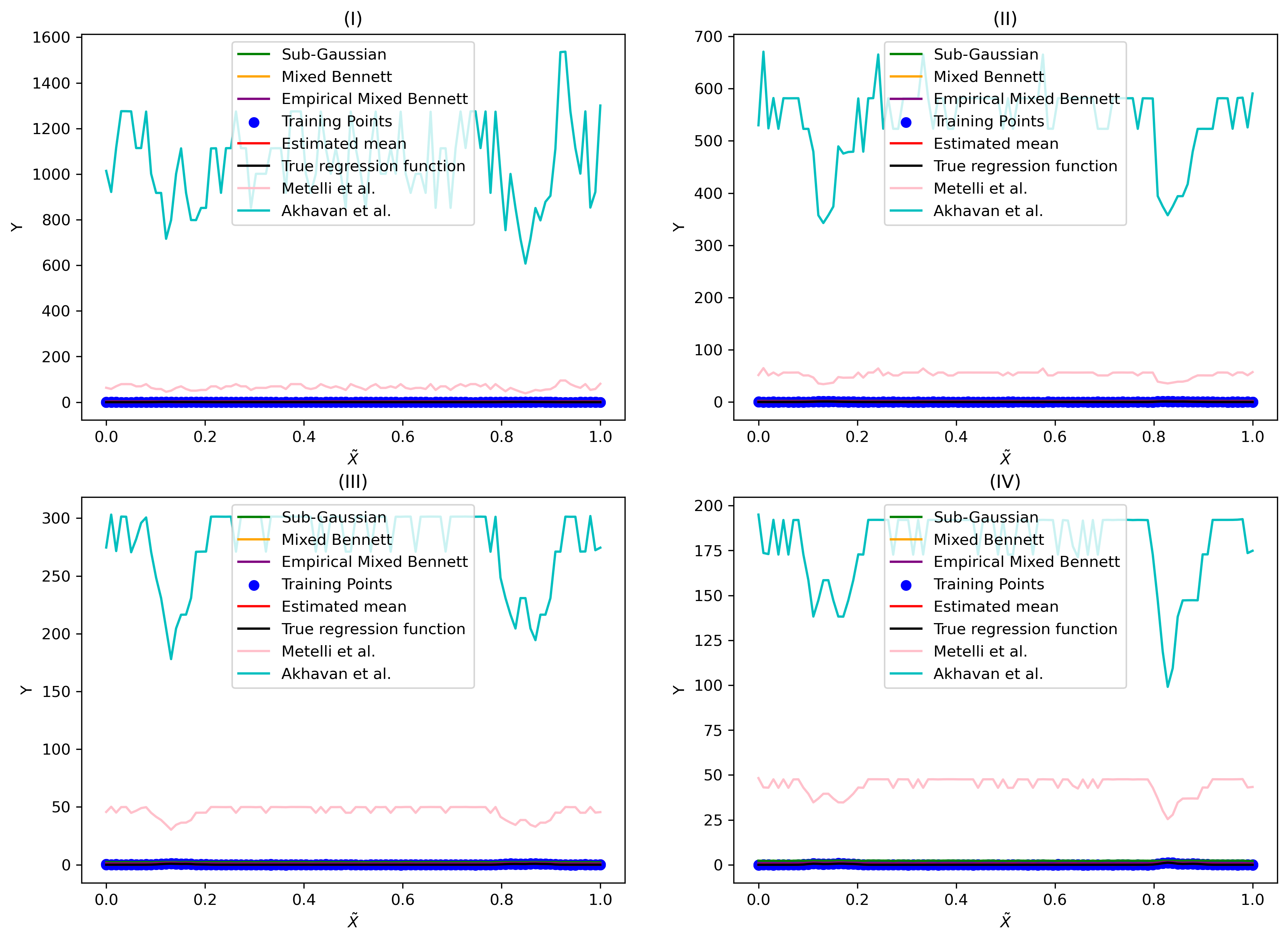}
  \caption{Illustration of the optimistic upper confidence bounds for the regression function after $500$ rounds using sub-Gaussian, mixed Bennett, empirical mixed Bennett inequalities, \citet[Theorem 3.2]{akhavan2025bernstein}, and \citet[Theorem 5.1]{metelli2025generalized} for noises following (I) a rescaled uniform distribution, (II) a rescaled ($5, 5$)-beta distribution, (III) a rescaled ($20, 20$)-beta distribution, and (IV) a rescaled ($50, 50$)-beta distribution. Training points are drawn following a UCB procedure, with the covariates $X_t = k(\cdot, \tilde X_t)$ illustrated in the original space (pre-embedded in the RKHS). }
  \label{fig:experiments_recent_work}
\end{figure}

\section{Limitations of our work} \label{section:limitations}

We discuss in this section the limitations of our contribution. For simplicity, let us focus on bounded random noises (Assumption~\ref{assumption:bounded_variance}) such that the conditional standard deviation $\sigma_t$ is constant and equal to $\sigma$. In such a setting, our Bernstein-type inequality from Theorem~\ref{corollary:fixed_bernstein} establishes a confidence interval with radius
\begin{align*}
  B \log \lp \frac{2}{\delta} \rp + C_n\sqrt{2  \log \lp \frac{2}{\delta}\rp}.
\end{align*}
In view of $C_n^2 \leq \sigma^2 \sum_{i \leq n} \|G_i\|^2 \leq 4 \sigma^2 \gamma_n(\rho)$, the dominating term of the above expression can be upper bounded by  
\begin{align*}
     2 \sigma \sqrt{2 \gamma_n(\rho)  \log \lp \frac{2}{\delta}\rp }.
\end{align*}
Furthermore, the sub-Gaussian concentration inequality from \citet{abbasi2011improved} yields a confidence interval with a radius that can be upper bounded by
\begin{align*}
    \frac{B}{2} \sqrt{2 \log \lp \frac{1}{\delta} \rp + 2 \gamma_n(\rho) }.
\end{align*}
Assuming that $\sigma \approx B$ for ease of comparison, note that both radii scale as $\bigo (\sqrt{\gamma_n(\rho)})$ with $n$ (this is precisely the term that is usually considered in the regret analyses for the bandit problem).

However, our inequality was obtained after optimizing for $\lambda$ for a given $n$, while the inequality from \citet{abbasi2011improved} uses a mixture argument of the analogous hyperparameter (which is a vector in their case, and they mix following a standard multivariate Gaussian distribution). It is well understood that mixing generally leads to confidence intervals that are inflated by a logarithmic factor of the pseudo-variance process in comparison to tightly optimized inequalities, see e.g. \citet[Section 3]{howard2021time} for a discussion. This begs the question of whether our inequalities could be improved by a logarithmic factor.

In order to address this question, let us consider the reduction of our problem to one dimension. We can think of the one-dimensional problem as the multivariate setting where all the directions are $X_t = r_t e_1$. For simplicity, we assume that these vectors have constant unit norms ($r_t = 1$), so $X_t = e_1$ for all $t$. In this case, 
\begin{align*}
    \ldba (\rho I + V_t)^{-1/2}M_t  \rdba = \frac{\lba \sum_{i \leq t} \epsilon_i \rba}{ \sqrt{\rho+t}}, \quad \ldba G_t\rdba = \frac{1}{ \sqrt{\rho+t}}.
\end{align*}
Observing that \(\sum_{i \leq n} \|G_i\|^2 = \sum_{i \leq n} \frac{1}{\rho + i} \asymp \log n\), we find that our self-normalized Bernstein inequality (Theorem~\ref{corollary:fixed_bernstein}) provides confidence radii scaling as \(\bigo(\sqrt{\log n})\). By contrast, applying the classical univariate Bernstein inequality to \(|\sum_{i \leq t} \epsilon_i|\) yields radii of order \(\sqrt{n}\); after division by \(\sqrt{\rho + n}\), these become \(\bigo(1)\). Thus, our inequalities are loose by a logarithmic factor, at least in this scenario. This extra logarithmic factor can be directly recognized in the supermartingale construction, which in this one dimensional setting reduces to 
\begin{align*}
     S_t = \cosh \lp \lambda \frac{\lba \sum_{i \leq t} \epsilon_i \rba }{ \sqrt{\rho+t}}  \rp \exp \lp - \sigma^2\psi_{P, B}(\lambda)\sum_{i \leq t} \frac{1}{\rho + i}  \rp,
\end{align*}
where $\psi_{P, B}(\lambda) = \frac{e^{\lambda B} - \lambda B - 1}{B^2}$. The  term $\sum_{i \leq t} 1/(\rho + i)$ being $\bigo(\log(t))$, as opposed to $\bigo(1)$, is what causes the looseness in the final concentration inequality.

One may wonder whether this looseness stems from the overall approach (i.e., looking for a nonnegative supermartingale that is already self-normalized), or rather a weak technical analysis of it. We shall argue for the former. In order to see this, let us consider the simpler (only involving $\exp$, not $\cosh$) one-dimensional supermartingale construction 
\begin{align*}
     S_t^+ &= \exp \lp \lambda \frac{\sum_{i \leq t} \epsilon_i }{ \sqrt{\rho+t}}   - \sigma^2\psi_{P, B}(\lambda)\sum_{i \leq t} \frac{1}{\rho + i}  \rp.
\end{align*}
Such a nonnegative supermartingale is of the form 
\begin{align*}
    S_t(\psi, V_t) = \exp \lp \lambda \frac{\sum_{i \leq t} \epsilon_i }{ \sqrt{\rho+t}}   - \psi(\lambda)V_t\rp 
\end{align*}
with $\psi(\lambda) = \psi_{P, B}(\lambda)$, and $V_t = \sigma^2\sum_{i \leq t} \frac{1}{\rho + i}$.
We would like to find $\psi$ and $V_t$, such that $\psi$ is \textit{CGF-like}\footnote{A real valued function $\psi$ with domain $[0, \lambda_{\max})$ is called \textit{CGF-like} if it is strictly convex and twice continuously differentiable with $\psi(0) = \psi'(0_+) = 0$ and $\sup_{\lambda \in [0, \lambda_{\max})}\psi(\lambda) = \infty$. See \citet[Section 2.1]{howard2020time} for a detailed presentation.} and $V_t$ is nonnegative and $\bigo(1)$.
If there do not exist such $\psi$ and $V_t$, then we can conclude that the limitations of our inequalities stem indeed from the approach itself. To see that this is the case, observe that
\begin{align*}
    \frac{S_t(\psi, V_t)}{S_{t-1}(\psi, V_{t-1})} =  \exp \lp \lambda \lp \frac{1}{ \sqrt{\rho+t}}-\frac{1}{ \sqrt{\rho+t+1}} \rp \sum_{i \leq {t-1}} \epsilon_i + \lambda \frac{\epsilon_t }{ \sqrt{\rho+t}}  - \psi(\lambda)(V_t-V_{t-1})\rp.
\end{align*}
Note that
\begin{align*}
    \lp \frac{1}{ \sqrt{\rho+t}}-\frac{1}{ \sqrt{\rho+t+1}} \rp \sum_{i \leq {t-1}} \epsilon_i
\end{align*}
can be arbitrarily small (on an event with non-zero probability), and hence in order to obtain the supertingale condition
\begin{align*}
    \frac{\E_{t-1} S_t(\psi, V_t)}{S_{t-1}(\psi, V_{t-1})} \leq 1,
\end{align*}
we ought to have
\begin{align*}
    \E_{t-1}\exp \lp  \lambda \frac{\epsilon_t }{ \sqrt{\rho+t}}  - \psi(\lambda)(V_t-V_{t-1})\rp \lessapprox 1.
\end{align*}
If $\epsilon_t$ is sub-$\psi$ with pseudo-variance $\sigma^2$, i.e.
\begin{align*}
    \E_{t-1} \exp \lp   \lambda \epsilon_t  - \psi(\lambda)\sigma^2 \rp \leq 1,
\end{align*}
then $\epsilon_t / \sqrt{\rho + t}$ is sub-$\psi$ with pseudo-variance $\sigma^2/(\rho + t)$. Consequently, $V_t - V_{t-1}$ can be taken as  $\sigma^2/(\rho + t)$. However, this implies that $V_t$ is $\omega(1)$. Hence, the limitations of our inequalities seem to stem from the approach itself, not than a loose technical analysis. Investigating whether these inequalities can be further refined through an alternative approach remains an open direction of research.

%As discussed in the main body of the project, our nonnegative supermartingale construction builds directly on the properties of the norm. This is in stark contrast to the majority of the current works, that build on univariate-based supermartingales (as they project the vector into any arbitrary dimension), and then they take a mixture \citep{abbasi2011improved} or union bound\citep{whitehouse2023time} over such univariate projections.  This approach comes with some limitations. For simplicity, let us 